\newif\ifarXiv         
\newif\ifjournal        
  \journalfalse       \arXivtrue      

 \ifarXiv 
\documentclass[times,11pt]{article} 
\usepackage[small,compact]{titlesec}
\usepackage{amsmath,amssymb,amsfonts,mathabx,setspace}
\usepackage{graphicx,caption,epsfig,subcaption,epsfig,wrapfig,sidecap}
\usepackage{url,color,verbatim,algorithmicx}
\usepackage[sort,compress]{cite}
\usepackage{soul}
\usepackage{xspace}

\usepackage[noend]{algpseudocode}
\usepackage[ruled,boxed]{algorithm}

\newcounter{multifig} 

\newcommand{\figcaption}[1]
{\stepcounter{multifig}
\addcontentsline{lof}{figure}{\string\numberline {\arabic{multifig}}{\ignorespaces #1}} Figure \arabic{multifig}: #1}

\usepackage{multirow}

\usepackage[scaled]{helvet}
\usepackage[T1]{fontenc}
\usepackage[bookmarks=true, bookmarksnumbered=true, colorlinks=true,   pdfstartview=FitV,
linkcolor=blue, citecolor=blue, urlcolor=blue]{hyperref}
\usepackage{multirow, makecell, colortbl, hhline}
\usepackage{booktabs} 
\usepackage{enumitem} 
\usepackage{etoc}          
\usepackage[dvipsnames]{xcolor}

\def\R{\mathbb{R}}

\def\E{\mathbb{E}}
\newcommand{\Ebracket}[1]{\mathbb{E}\left[{#1}\right]}

\def\mH{\mathcal{H}}
\def\errfcnl{\mathcal{E}}

\def\rhoT{\widebar \rho_T}

\def\ftrue{\ensuremath{f_*}\xspace}
\def\momf{\ensuremath{\xi}\xspace}
\def\ymin{y_{\min}}
\def\ymax{y_{\max}}

\def\L{\mathcal{L}}
\def\sumLavg{\frac{1}{L}\sum_{l=1}^L}
\def\intTavg{\frac{1}{T}\int_0^T}
\newcommand{\argmin}[1]{\underset{#1}{\operatorname{arg}\operatorname{min}}\;}

\newcommand{\supp}[1]{\text{supp}(#1)}

\newcommand{\innerp}[1]{\langle{#1}\rangle}
\newcommand{\floor}[1]{\lfloor{#1}\rfloor}

\definecolor{mygrey}{gray}{0.75}

\def\XXint#1#2#3{{\setbox0=\hbox{$#1{#2#3}{\int}$ }
\vcenter{\hbox{$#2#3$ }}\kern-.6\wd0}}

\newtheorem{theorem}{Theorem}

\newtheorem{assumption}[theorem]{Assumption}

\newtheorem{definition}[theorem]{Definition}
\newtheorem{myexample}[theorem]{Example}

\newtheorem{lemma}[theorem]{Lemma}

\newtheorem{remark}[theorem]{Remark}

\newenvironment{proof}[1][Proof]{\noindent\textbf{#1.} }{\ \rule{0.5em}{0.5em}}

\numberwithin{equation}{section}
\numberwithin{theorem}{section}

\title{Unsupervised learning of observation functions in \\  state-space models  by nonparametric moment methods}
\usepackage{authblk} 

\author[1]{Qingci An}
\author[2,5]{Yannis Kevrekidis}
\author[3]{Fei Lu}
\author[2,4]{Mauro Maggioni}
\affil[1,3,4]{\footnotesize  Department of Mathematics, Johns Hopkins University, Baltimore, MD 21218, USA \newline
\href{qan2@jhu.edu}{qan2@jhu.edu},
\href{feilu@math.jhu.edu} {feilu@math.jhu.edu}, 
\href{mmaggio4@jh.edu}{mmaggio4@jh.edu} 
}
\affil[2]{\footnotesize  Department of Applied Mathematics and Mathematics, Johns Hopkins University, Baltimore, MD 21218, USA
\href{yannisk@jhu.edu}{yannisk@jhu.edu}
}

\date{}
\fi

\usepackage{enumitem}
\usepackage{xcolor}
\usepackage{full page}

\begin{document}
\maketitle
\begin{abstract}
We investigate the unsupervised learning of non-invertible observation functions in nonlinear state-space models. 
Assuming abundant data of the observation process along with the distribution of the state process, we introduce a nonparametric generalized moment method to estimate the observation function via constrained regression. The major challenge comes from the non-invertibility of the observation function and the lack of data pairs between the state and observation. We address the fundamental issue of identifiability from quadratic loss functionals and show that the function space of identifiability is the closure of a RKHS that is intrinsic to the state process. 
Numerical results show that the first two moments and temporal correlations, along with upper and lower bounds, can identify functions ranging from piecewise polynomials to smooth functions, leading to convergent estimators. The limitations of this method, such as non-identifiability due to symmetry and stationarity, are also discussed. 
\end{abstract}

\ifjournal
\begin{tcbverbatimwrite}{tmp_\jobname_abstract.tex}

\end{tcbverbatimwrite}
\input{tmp_\jobname_abstract.tex}
\fi 

\ifarXiv
\vspace{3mm}
 \tableofcontents
\fi


\section{Introduction}
We consider the following state-space model for $(X_t, Y_t)$ processes in $\mathbb{R}\times \R$:
\begin{align}
& \text{State model:} & dX_t & = a(X_t)dt + b(X_t) dB_t,
& \text{ with } &a, b \text{are known};  & \label{eq:SM} \\
& \text{Observation model:} & Y_t & = \ftrue(X_t),             & \text{ with } & \ftrue   \text{ unknown}.  & \label{eq:OM} 
 \end{align} 
Here $(B_t)$ is the standard Brownian motion, the drift function $a(x)$ and the diffusion coefficient $b(x)$ are given, satisfying the linear growth and global Lipschitz conditions. We assume that the initial distribution of $X_{t_0}$ is given. Thus, the state model is known, in other words, the distribution of the process $(X_t)$ is known.

Our goal is to estimate the unknown observation function $\ftrue$ from data consisting of an ensemble of trajectories of the process $Y_t$, denoted by $\{Y_{t_0:t_L}^{(m)}\}_{m=1}^M$, where $m$ indexes trajectories, $t_0<\dots<t_L$ are the times at which the observations are made. 
In particular, there are no pairs  $(X_t,Y_t)$ being observed, so in the language of machine learning this may be considered an unsupervised learning problem. 
A case of particular interest in the present work is when the observation function $\ftrue$ is nonlinear and {\em non-invertible}, and it is within a large class of functions, including  smooth functions but also, for example, piecewise continuous functions.

We estimate the observation function $\ftrue$ by matching generalized moments, while constraining the estimator to a suitably chosen finite-dimensional hypothesis (function) space, whose dimension depends on the number of observations, in the spirit of nonparametric statistics. We consider both first- and second-order moments, as well as temporal correlations, of the observation process. The estimator minimizes the discrepancy between the moments over hypothesis spaces spanned by B-spline functions, with upper and lower pointwise constraints estimated from data. 
The method we propose has several significant strengths:
\begin{itemize}
\item the generalized moments do not require the invertibility of the observation function \ftrue; 
\item low-order generalized moments tend to be robust to additive observation noise; 
\item generalize moments avoid the need of local constructions, since they depend on the entire distribution of the latent and observed processes; 
\item our nonparametric approach does not require {\em a priori} information about the observation function, and, for example, it can deal with both continuous and discontinuous functions;  
\item the method is computationally efficient because the moments need to be estimated only once, and their computation is easily performed in parallel. 
\end{itemize}
We note that the method we propose readily extends to multivariate state models, with the main statistical and computational bottlenecks coming from the curse-of-dimensionality in the representation and estimation of a higher-dimensional $f_*$ in terms the basis functions.

The problem we are considering has been studied in the contexts of nonlinear system identification \cite{billings_NonlinearSystem2013,ljung1998system}, filtering and data assimilation \cite{cappe_InferenceHidden2005,LSZ15}, albeit typically when observations are in the form of one, or a small number of, long trajectories, and in the case of an invertible or smooth observations function \ftrue.
The estimation of the unknown observation function and of the latent dynamics from unlabeled data has been considered in \cite{hafner2019_LearningLatent,kaiser2020_ModelBasedReinforcement,gelada2019_DeepMDPLearning,moosmuller2019_GeometricApproach} and references therein.
Inference for state-space models (SSMs) has been widely studied; most classical approaches focus on estimating the parameters in the SSM from a single trajectory of the observation process, by expectation-maximization methods maximizing the likelihood, or Bayesian approaches \cite{billings_NonlinearSystem2013,ljung1998system,cappe_InferenceHidden2005,kantas_OverviewSequential2009,ghosh2014_BayesianInference}, with the recent studies estimating the coefficients in a kernel representation  \cite{tobar2015_UnsupervisedStateSpace}  or the coefficients of a pre-specified set of basis functions \cite{svensson2017_FlexibleState}.

Our framework combines nonparametric learning  \cite{Gyorfi06,cucker2007learning} with the generalized moments method, that is mainly studied in the setting of parametric inference \cite{sorensen_EstimatingFunctions2012,prakasarao_StatisticalInference1988,pokern_ParameterEstimation2009}. 
We study the identifiability of the observation function from first-order moments, and show that the first-order generalized moments can identify the function in the $L^2$ closure a reproducing kernel Hilbert space (RKHS) that is intrinsic to the state model. As far as we know, this is the first result on the function space of identifiability for nonparametric learning of observation functions in SSMs.

When the observation function is invertible, its unsupervised regression is investigated \cite{rahimi07unsupervised} by maximizing the likelihood for high-dimensional data. However, in many applications, particularly those involving complex dynamics, the observation functions are non-invertible, for example they are projections or nonlinear non-invertible transformations (e.g.,$f(x) = |x|^2$ with $x\in \R^d$).  As a consequence, the resulting observed process may have discontinuous or singular probability densities \cite{jeffrey2018_HiddenDynamics,guglielmi2015_ClassificationHidden}. In \cite{moosmuller2019_GeometricApproach}, it has been shown empirically that delayed coordinates with principal component analysis may be used to estimate the dimension of the hidden process, and diffusion maps \cite{DiffusionPNAS} may yield a diffeomorphic copy of the observation function.

The remainder of the paper is organized as follows. We present the nonparametric generalized moments method in Section \ref{sec2}. In Section \ref{sec3} we study the identifiability of the observation function from first-order moments, and show that the function spaces of identifiability are RKHSs intrinsic to the state model. We present numerical examples to demonstrate the effectiveness, and limitations, of the proposed method in Section \ref{sec4}. Section \ref{sec5} summarizes this study and discusses directions of future research; we review the basic elements about RKHSs in Appendix \ref{sec:appendixA}. 

\section{Non-parametric regression based on generalized moments}\label{sec2}

Throughout this work, we focus on discrete-time observations of the state-space model \eqref{eq:SM}--\eqref{eq:OM},  because data in practice are discrete in time, and the extension to continuous time trajectories is straightforward. We thereby suppose that the data is in the form $\{Y_{t_0:t_L}^{(m)} \}_{m=1}^M$, with $m$ indexing multiple independent trajectories, observed at the vector $t_0:t_L$ of discrete times $(t_0,\cdots,t_L)$.

\subsection{Generalized moments method}
We estimate the observation function $\ftrue$ by the generalized moment method (GMM) \cite{sorensen_EstimatingFunctions2012,prakasarao_StatisticalInference1988,pokern_ParameterEstimation2009},
searching for an observation function $\widehat f$, in a suitable finite-dimensional hypothesis (function) space, such that the moments of functionals of the process $(\widehat f(X_t))$ are close to the empirical ones (computed from data) of $\ftrue(X_t)$.

We consider ``generalized moments'' in the form $\Ebracket{\momf(Y_{t_0:t_L}) }$, where $\momf:\R^{L+1}\to \R^K$ is a functional of the trajectory $Y_{t_0:t_L}$. For example,  the functional $\momf$ can be $\momf(Y_{t_0:t_L}) = [Y_{t_0:t_L},  Y_{t_0}Y_{t_1},\ldots,Y_{t_{L-1}}Y_{t_{L}}]\in \R^{2L+1}$, in which case $\Ebracket{\momf(Y_{t_0:t_L}) }=  \left[ \Ebracket{Y_{t_0:t_L}}, \Ebracket{Y_{t_1}Y_{t_2}},\ldots,\Ebracket{Y_{t_{L-1}}Y_{t_{L}} }\right]$ is the vector of the first moments and of temporal correlations at consecutive observation times.  The empirical generalized moments $\xi$ are computed from data by Monte Carlo approximation:
 \begin{equation}\label{eq:exp_Y}
  \Ebracket{\momf(Y_{t_0:t_L}) } \approx E_{M}[{\momf(Y_{t_0:t_L}) }] := \frac{1}{M} \sum_{m=1}^M \momf(Y_{t_0:t_L}^{(m)}),
  \end{equation} 
which converges at the rate $M^{-1/2}$ by the Central Limit Theorem, since the $M$ trajectories are independent. Meanwhile, since the state model (hence the distribution of the state process) is known, for any putative observation function $f$, we approximate the moments of the process $(f(X_t)))$ by simulating $M'$ independent trajectories of the state process $(X_t)$: 
\begin{equation}\label{eq:exp_X}
 \Ebracket{\momf(f(X)_{t_0:t_L}) } \approx 
 \frac{1}{M'} \sum_{m=1}^{M'} \momf(f(X)_{t_0:t_L}^{(m)})\,.
 \end{equation} 
Here, with some abuse of notation, $f(X)_{t_0:t_L}^{(m)}:=(f(X_{t_0}^{(m)}),\dots,f(X_{t_L}^{(m)}))$.
The number $M'$ can be as large as we can afford from a computational perspective. In what follows, since $M'$ can be chosen large -- only subject to computational constraints -- we consider the error in this empirical approximation negligible and work with $\Ebracket{\momf(f(X)_{t_0:t_L}) }$ directly. 

We estimate the observation function $\ftrue$ by minimizing a notion of discrepancy between these two empirical generalized moments:
\begin{align}\label{eq:distGMM}
\widehat f = \argmin{f\in \mH} \errfcnl^{M}(f), \quad \text{where }\errfcnl^{M}(f):=  \mathrm{dist}\left(E_{M}[{\momf(Y_{t_0:t_L}) }], \Ebracket{\momf(f(X)_{t_0:t_L}) } \right)^2 ,
\end{align}
where $f$ is restricted to some suitable hypothesis space $\mH$, and $\mathrm{dist}(\cdot,\cdot)$ is a proper distance between the moments to be specified later. We choose $\mH$ to be a subset of an $n$-dimensional function space, spanned by basis functions $\{\phi_i\}$, within which we can write $\widehat f =\sum_{i=1}^n \widehat c_i\phi_i$.
By the law of large numbers, $\errfcnl^M(f)$ tends almost surely to $\errfcnl(f):=\mathrm{dist}\left(\Ebracket{\momf(Y_{t_0:t_L}) }, \Ebracket{\momf(f(X)_{t_0:t_L}) } \right)^2$. 

It is desirable to choose the generalized moment functional $\momf$ and the hypothesis space $\mH$ so that the minimization in \eqref{eq:distGMM} can be performed efficiently.
We select the functional $\momf$ so that the moments $\Ebracket{\momf(f(X)_{t_0:t_L}) }$, for  $f=\sum_{i=1}^n c_i\phi_i$, can be efficiently evaluated for all $(c_1,\ldots,c_n)$. To this end, we choose linear functionals or low-degree polynomials, so that we only need to compute the moments of the basis functions once, and use these moments repeatedly during the optimization process, as discussed in Section \ref{sec:lossFn}. 
The selection of the hypothesis space is detailed in Section \ref{sec:basisFn_dim}.

\subsection{Loss functional and estimator}\label{sec:lossFn}
The generalized moments we consider include the first and the second moments, as well as the one-step temporal correlation: we let $\momf(Y_{t_0:t_L}):=  \left( Y_{t_0:t_L}, Y_{t_0:t_L}^2, Y_{t_0}Y_{t_1},\ldots,Y_{t_{L-1}}Y_{t_{L}} \right)  \in \R^{3L+2}$. The loss functional in \eqref{eq:distGMM} is then chosen in the following form: for weights $w_1,\dots,w_3>0$, 
\begin{equation}\label{eq:errfnl_12corr}
\begin{aligned}
\errfcnl(f)  := &  w_1  \underbrace{\sumLavg \big|\E[f(X_{t_l})] - \E[Y_{t_l}]|^2}_{\errfcnl_1(f)} + w_2\underbrace{\sumLavg \left|\E[f(X_{t_l})^2] - \E[Y_{t_l}^2]\right|^2}_{\errfcnl_2(f)}  \\
& + w_3 \underbrace{\sumLavg \left|\E[f(X_{t_l})f(X_{t_{l-1}})] - \E[Y_{t_l}Y_{t_{l-1}}]\right|^2 }_{\errfcnl_3(f)} \,.
\end{aligned}
\end{equation}

Let the hypothesis space $\mH$ be a subset of the span of a linearly independent set $\{\phi_i\}_{i=1}^n$, which we specify in the next section. For $f =\sum_{i=1}^n c_i  \phi_i \in \mH$, we can write the loss functionals $\errfcnl_1(f)$  in \eqref{eq:errfnl_12corr} as
\begin{align}\label{eq:errfcnl1}
      \errfcnl_1(f) & =  \sumLavg  \bigg| \sum_{i=1}^n c_i \Ebracket{\phi_i(X_{t_l})}  - \Ebracket{Y_{t_l}} \bigg|^2   = c^\top \overline{A}_{1} c - 2c^\top \overline{b}_{1} + \Tilde{b}_{1}, 
\end{align}
where $\Tilde{b}_{1}:=   \sumLavg  \Ebracket{Y_{t_l}}^2$, and the matrix $\overline{A}_{1} $ and the vector $\overline{b}_{1} $ are given by  
        \begin{equation}\label{eq:Ab1}
        \begin{aligned}
        \overline{A}_{1}(i,j) &:=  \sumLavg \underbrace{ \Ebracket{\phi_i(X_{t_l})}\Ebracket{\phi_j(X_{t_l})}}_{A_{1,l}(i,j)}, \quad   \overline{b}_{1}(i) :=  \sumLavg  \underbrace{ \Ebracket{\phi_i(X_{t_l}) } \Ebracket{Y_{t_l}} }_{b_{1,l}(i)}. 
        \end{aligned}
    \end{equation}
Similarly, we can write $\errfcnl_2(f)$ and $\errfcnl_3(f)$ in \eqref{eq:errfnl_12corr} as
\begin{equation}\label{eq:errfcnl2corr}
\begin{aligned}
      \errfcnl_2(f)& =  \sumLavg \bigg| \sum_{i=1}^n c_ic_j \underbrace{ \Ebracket{ \phi_i \phi_j (X_{t_l})} }_{A_{2,l}(i,j)}  - \underbrace{\Ebracket{Y_{t_l}^2} }_{b_{2,l}} \bigg|^2.  \\
          \errfcnl_3(f) &=  \sumLavg  \bigg|\sum_{i=1}^n c_ic_j \underbrace{ \Ebracket{\phi_i (X_{t_{l-1}}) \phi_j (X_{t_l})} }_{A_{3,l}(i,j)}  - \underbrace{  \Ebracket{Y_{t_{l-1}} Y_{t_l}} }_{b_{3,l}} \bigg|^2.  
\end{aligned}
\end{equation}
Thus, with the above notations in \eqref{eq:Ab1}-\eqref{eq:errfcnl2corr}, the minimizer of the loss functional $\errfcnl(f)$ over $\mH$ is 
\begin{equation}\label{eq:minimizer1}
\begin{aligned}
\widehat f_\mH & := \sum_{i=1}^n \widehat c_i \phi_i, \quad 
 \widehat c := \argmin{c\in \R^n \text{ s.t. } \sum_{i=1}^n c_i  \phi_i \in \mH} \errfcnl(c), \, \text{ where } \\
\errfcnl(c)  & := w_1[ c^\top \overline{A}_{1} c - 2c^\top \overline{b}_{1} + \Tilde{b}_1] + \sum_{k=2}^3 w_k \sumLavg \left|  c^\top A_{k,l} c - b_{k,l} \right|^2.
\end{aligned}
\end{equation}
Here, with an abuse of notation, we denote $\errfcnl(\sum_{i=1}^n c_i  \phi_i)$ by $\errfcnl(c)$.

In practice, with data $\{Y_{[t_1:t_N]}^{(m)} \}_{m=1}^M$, we approximate the expectations involving the observation process $(Y_t)$ by the corresponding empirical means as in \eqref{eq:exp_Y}. Meanwhile, we approximate the expectations involving the state process $(X_t)$ by Monte Carlo as in \eqref{eq:exp_X}, using $M'$ trajectories. We assume that the sampling errors in the expectations of $(X_t)$, i.e. in the terms $\{A_{k,l}\}_{k=1}^3$, are negligible, since the basis $\{\phi_i\}$ can be chosen to be bounded functions (such as B-spline polynomials) and $M'$ can be as large as we can afford. We approximate $\{b_{k,l}\}_{k=1}^3$ by their empirical means $\{b_{k,l}^M\}_{k=1}^3$: 
\begin{alignat}{4}
&b_{1,l}(i) &&=  \Ebracket{\phi_i(X_{t_l})}\Ebracket{Y_{t_l}} &&\approx  \Ebracket{\phi_i(X_{t_l})} \frac{1}{M} \sum_{m=1}^M Y_{t_l}^{(m)}  &&=:b_{1,l}^M(i)\,, \\
&b_{2,l} &&=  \Ebracket{ |Y_{t_l}|^2} &&\approx  \frac{1}{M} \sum_{m=1}^M |Y_{t_l}^{(m) }|^2  &&=:b_{2,l}^M\,, \\
&b_{3,l} &&=  \Ebracket{ Y_{t_{l-1}} Y_{t_l} } &&\approx  \frac{1}{M} \sum_{m=1}^M Y_{t_{l-1}}^{(m)} Y_{t_l}^{(m) }  &&=:b_{3,l}^M\,. 
\label{eq:bkl}
\end{alignat}
Then, with $\overline{b}_{1}^M= \frac{1}{L}\sum_{l=1}^L b_{1,l}^M$ and $\widetilde{b}_1^M = \frac{1}{LM}\sum_{l=1}^L\sum_{m=1}^M\left( Y_{t_l}^{(m)}\right)^2 $, the estimator from data is 
\begin{equation}\label{eq:errfcnl-M}
\begin{aligned}
\widehat f_{\mH,M} &  = \sum_{i=1}^n \widehat c_i \phi_i, \qquad 
 \widehat c = \argmin{c\in \R^n \text{ s.t. } \sum_{i=1}^n c_i  \phi_i \in \mH} \errfcnl^M(c),  \, \text{ where }  \\ 
\errfcnl^M(c)  & =w_1[ c^\top \overline{A}_{1} c - 2c^\top \overline{b}_{1}^M + \widetilde{b}_1^M]+ \sum_{k=2}^3 w_k \sumLavg \left|  c^\top A_{k,l} c - b_{k,l}^M \right|^2.
\end{aligned}
\end{equation}
The minimization of $\errfcnl^M(c) $ can be performed with iterative algorithms, with each optimization iteration, with respect to $c$, performed efficiently since the data-based matrices and vectors, $\overline{A}_{1},\overline{b}_{1}^M$ and $ \{ A_{k,l},b_{k,l}^M\}_{k=2}^3$, only need to be computed once. 
The main source of sampling error is the empirical approximation of the moments of the process $(Y_t)$. 
We specify the hypothesis space in the next section and provide a detailed algorithm for the computation of the estimator in Section \ref{sec:algorithm}.  

\begin{remark}[Moments involving It\^o's formula] 
When the data trajectories are continuous in time (or when they are sampled with a high frequency in time), we can utilize additional moments from It\^o's formula. Recall that for $f\in C^2_b$,  applying It\^o  formula for the diffusion process in \eqref{eq:SM}, we have
\[ f(X_{t+\Delta t}) -f(X_{t})  = \int_t^{t+\Delta t}\nabla f\cdot b(X_s)dW_s +\int_t^{t+\Delta t} \L f(X_s)ds,
\]
where the operator $\L$ is 
\begin{equation}\label{eq:Loperator}
\L f = \nabla f \cdot a +\frac{1}{2} Hess(f):b^\top b.
\end{equation}
Hence, $\Ebracket{\Delta Y_{t_l}} = \Ebracket{\L \ftrue(X_{t_{l-1}})}\Delta t + o(\Delta t)$, where $\Delta Y_{t_l} = Y_{t_{l}}- Y_{t_{l-1}}$. 
Thus, when $\Delta t$ is small, we can consider matching the generalized moments \begin{align}\label{eq:errfcnl4}
      \errfcnl_4(f) & =  \sumLavg  \bigg| \Ebracket{\L f(X_{t_{l-1}})} \Delta t  - \Ebracket{\Delta Y_{t_l}} \bigg|^2.
\end{align}
Similarly, we can further consider the generalized moments  $\Ebracket{Y_t\Delta Y_t}$ and $\mathrm{Var}{(\Delta Y_t)}$ and the corresponding quartic loss functionals. Since they require the moments of the first- and second-order derivatives of the observation function, they are helpful when the observation function is smooth with bounded derivatives. 
\end{remark}

\subsection{Hypothesis space and optimal dimension}\label{sec:basisFn_dim}
We let the hypothesis space $\mH$ be a class of bounded functions in $\mathrm{span}\{\phi_i\}_{i=1}^n$, 
\begin{equation} \label{eq:mH_with_bds}
\mH := \{ f\,:\,f= \sum_{i=1}^n c_i\phi_i : \ymin\leq f(x)\leq \ymax \text{ for all } x\in \supp{\rhoT}\},
\end{equation}
where the basis functions $\{\phi_i\}$ are to be specified below, and the empirical bounds
\[\ymin := \min\{Y_{t_l}^{(m)}\}_{l,m=1}^{L,M}, \quad \ymax := \max\{Y_{t_l}^{(m)}\}_{l,m=1}^{L,M}
\]
aim to approximate the upper and lower bounds for $\ftrue$. 
Note that the hypothesis space $\mH$ is a bounded convex subset of the linear space $\mathrm{span}\{\phi_i\}_{i=1}^n$. 
While the pointwise bound constraints are for all $x$, in practice, for efficient computation, we apply these constraints at representative points, for example at the mesh-grid points used when the basis functions are piecewise polynomials. One may apply stronger constraints, such as requiring time-dependent bounds to hold at all times: $\ymin(t)\leq \sum_i^n c_if_i(x) \leq \ymax(t)$ for each time $t$, where $\ymin(t)$ and $\ymax(t)$ are the minimum and maximum of the data set $\{Y_t^{(m)}\}_{m=1}^M$. 

\paragraph{Basis functions.}
As basis functions $\{\phi_i\}$ for the subspace containing $\mH$ we choose B-spline basis consisting of piecewise polynomials (see Appendix \ref{sec:append_HypoSpace} for details). 
To specify the knots of B-spline functions, we introduce a density function $\rhoT^L$, which is the average of the probability densities $\{p_{t_l}\}_{l=1}^L$ of $\{X_{t_l}\}_{l=1}^L$: 
\begin{equation}\label{eq:rho}
\rhoT^L(x) = \sumLavg p_{t_l}(x)\quad \xrightarrow[]{L\to \infty} \,  \rhoT(x)=\intTavg p_t(x)dt, 
\end{equation}
when $t_L=T$ and $\max_{1\leq l \leq L} |t_{l}-t_{l-1}|\to 0$. 
Here $\rhoT^L$ (and its continuous time limit $\rhoT(x)$) describes the intensity of visits to the regions explored by the process $(X_t)$.  
The knots of the B-spline function are from a uniform partition of $[R_{min}, R_{max}]$, the smallest interval enclosing the support of $\rhoT^L$.  
Thus, the basis functions $\{\phi_i\}$ are piecewise polynomials with knots adaptive to the state model which determines $\rhoT^L$.

\paragraph{Dimension of the hypothesis space.} It is important to select a suitable dimension of the hypothesis space to avoid under- or over-fitting. We select the dimension in two steps. First, we introduce an algorithm, namely \emph{Cross-validating Estimation of Dimension Range} (CEDR), to estimate the range of the dimension from the quadratic loss functional $\errfcnl_1$. Its main idea is to avoid the sampling error amplification due to an unsuitably large dimension. The sampling error is estimated from data by splitting the data into two sets. Then, we select the optimal dimension that minimizes the 2-Wasserstein distance between the measures of data and prediction. See Appendix \ref{sec:append_HypoSpace} for details. 

\subsection{Algorithm}\label{sec:algorithm}
We summarize the above method of nonparametric regression with generalized moments in Algorithm \ref{alg:main}. It minimizes a quartic loss function with the upper and lower bound constraints, 
and we perform the optimization with multiple initial conditions (see Appendix \ref{sec:multiIC} for the details).  

\begin{algorithm}[H]
{\small
\caption{Estimating the observation function by nonparametric generalized moment methods}\label{alg:main}
\begin{algorithmic}[1]
\Require{The  state model and data $\{Y_{t_0:t_L}^{(m)} \}_{m=1}^M$ consisting of multiple trajectories of the observation process.}
\Ensure{Estimator $\widehat f$.}
\State  Estimate the empirical density $\rhoT$ in \eqref{eq:rho}  and find its support $[R_{min}, R_{max}]$. 
\State Select a basis type, Fourier  or B-spline, with an estimated dimension range $[1,N]$ (by Algorithm \ref{alg:DimensionRange}), and compute the basis functions as described in Section \ref{sec:basisFn_dim}. 
\For{$n =1:N$}
	 \State Compute the moment matrices in \eqref{eq:Ab1}-\eqref{eq:errfcnl2corr}  and the vectors  $b_{k,l}^M$ in \eqref{eq:bkl}. 	
	\State Find the estimator $\widehat c_n$ by optimization with multiple initial conditions. Compute and record the values of the loss functional and the 2-Wasserstein distances.  
\EndFor 
\State Select the optimal dimension $n^*$ (and degree if B-spline basis) that has the minimal 2-Wasserstein distance in \eqref{eq:W2}. Return the estimator $\widehat f = \sum_{i = 1}^{n^*} c^i_{n^*} \phi_i$.
\end{algorithmic}
}
\end{algorithm}

\paragraph{Computational complexity} The computational complexity is driven by the construction of the normal matrix and vectors and the evaluation of the 2-Wasserstein distances, which require computations of order $\mathrm{O}(n^2 LM)$ and $\mathrm{O}(n LM)$, respectively. Thus, the total computational complexity is of order  $\mathrm{O}((n^2+n)LM)$.

\subsection{Tolerance to noise in the observations}\label{sec:noisyAlg}
The (generalized) moment method can tolerate large additive observation noise if the distribution of the noise is known. The estimation error caused by the noise is at the scale of the sampling error, which is negligible when the sample size is large. 

More specifically, suppose that we observe $\{Y_{t_0:t_L}^{(m)} \}_{m=1}^M$ from the observation model 
\begin{equation}\label{eq:obsModel_noisy}
Y_{t_l} = \ftrue(X_{t_l}) + \eta_{t_l},
\end{equation}
 where $\{\eta_{t_l}\}$ is sampled from a process $(\eta_t)$ that is independent of $(X_t)$ and has moments 
 \begin{equation}\label{eq:xi}
 \E[\eta_t] = 0, \quad C(s,t) = \E[\eta_t \eta_s], \text{ for } s,t\geq 0. 
 \end{equation}
 A typical example is when $\eta$ being identically distributed independent Gaussian noise $\mathcal{N}(0,\sigma^2)$, which gives $C(s,t) = \sigma^2\delta(t-s)$.
 
The algorithm in Section \ref{sec2} applies the noisy data with only a few changes. First, note that the loss functional in \eqref{eq:errfnl_12corr} involves only the moments $\E[Y_t]$, $\E[Y_t^2]$ and $\E[Y_{t_l}Y_{t_{l-1}}]$, which are moments of $\ftrue(X_t)$. When $Y_t$ in \eqref{eq:obsModel_noisy} has observation noise specified above, we have
\begin{align*}
\E[\ftrue(X_t)] & = \E[Y_t]  -\E[\eta_t] =  \E[Y_t];  \\
\E[\ftrue(X_t) \ftrue(X_s)] &= \E[Y_t Y_s]  -  \E[\eta_t\eta_s] =  \E[Y_tY_s] - C(t,s)
\end{align*}
for all $t,s\geq 0$. Thus, we only need to change the loss functional to be 
 \begin{equation}\label{eq:errfnl_12corr_noise}
\begin{aligned}
\errfcnl(f)  = &  w_1 \sumLavg \big|\E[f(X_{t_l})] - \E[Y_{t_l}]|^2 + w_2 \sumLavg \left|\E[f(X_{t_l})^2] - \E[Y_{t_l}^2]+ C(t,t)\right|^2  \\
& + w_3 \sumLavg \left|\E[f(X_{t_l})f(X_{t_{l-1}})] - \E[Y_{t_l}Y_{t_{l-1}} ]+ C(t,s)\right|^2. 
\end{aligned}
\end{equation}
Similar to \eqref{eq:errfcnl-M}, the minimizer of the loss functional can be then computed as 
\begin{equation}\label{eq:errfcnl-M_noise}
\begin{aligned}
\widehat f_{\mH,M} &  = \sum_{i=1}^n \widehat c_i \phi_i, \quad 
 \widehat c = \argmin{c\in \R^n \text{ s.t. } \sum_{i=1}^n c_i  \phi_i \in \mH} \errfcnl^M(c),  \, \text{ where }  \\
\errfcnl^M(c)  & =w_1[ c^\top \overline{A}_{1} c - 2c^\top \overline{b}_{1}^M + \widetilde{b}_1^M]+ w_2 \sumLavg \left|  c^\top A_{2,l} c - b_{2,l}^M + C(t_l,t_l)\right|^2 \\
& +  w_3 \sumLavg \left|  c^\top A_{3,l} c - b_{3,l}^M  + C(t_l,t_{l+1})\right|^2,
\end{aligned}
\end{equation}
where all the $A$-matrices and $b$-vectors are the same as before (e.g., in \eqref{eq:Ab1}--\eqref{eq:errfcnl2corr} and \eqref{eq:bkl}). 

Note that the observation noise introduces sampling errors through $b_{1}^M$, $b_{2,l}^M$ and $b_{2,l}^M$, which are at the scale $\mathrm{O}(\frac{1}{\sqrt{M}})$. Also, note the $A$-matrices are independent of the observation noise. Thus, the observation noise affects the estimator only through the sampling error at the scale $\mathrm{O}(\frac{1}{\sqrt{M}})$, the same as the sampling error in the estimator from noiseless data. 

\section{Identifiability}\label{sec3}

We discuss in this section the identifiability of the observation function by those loss functionals in the previous section. 
We show that $\errfcnl_1$, the quadratic loss functional based on the 1st-order moments in \eqref{eq:errfcnl1}, can identify the observation function in the $L^2(\rhoT^L)$-closure of a reproducing kernel Hilbert space (RKHS) that is intrinsic to the state model. 
In addition,  the loss functional $\errfcnl_4$ in \eqref{eq:errfcnl4} based on the It\^o formula, enlarges the function space of identifiability.
We also discuss, in Section \ref{sec:nonID}, some limitations of the loss functional $\errfcnl$ in \eqref{eq:errfnl_12corr_noise}, that combines the quadratic and quartic loss functionals; in particular, symmetry and stationarity may prevent us from identifying the observation function when using only generalized moments. 

The starting point is a definition of identifiability, which is a generalization of the uniqueness of minimizer of a loss function in parametric inference (see e.g., \cite[page 431]{BD91} and \cite{FY03}). 

\begin{definition}[Identifiability]\label{def:identifiability} 
 We say that the observation function $\ftrue$  is \emph{identifiable} by a data-based loss functional $\errfcnl$ on a function space $H$ if $\ftrue$ is the unique minimizer of $\errfcnl$ in $H$. 
\end{definition}
The identifiability consists of three elements: a loss functional  $\errfcnl$, a function space $H$, and a unique minimizer for the loss functional in $H$. When the loss functional is quadratic (such as $\errfcnl_1$ or $\errfcnl_4$), it has a unique minimizer in a Hilbert space iff its Frech\'et derivative is invertible in the Hilbert space; thus, the main task is to find such function spaces \cite{LLMTZ21,LangLu21,LLA22}. We will specify such function spaces for $\errfcnl_1$ and/or $\errfcnl_4$ in Section \ref{sec:ID-RKHS}. 
We note that these function spaces do not take into account the constraints of upper and lower bounds, which generically lead to minimizers near or at the boundary of the constrained set. This consideration applies also to the piecewise quadratic functionals $\errfcnl_2$ and $\errfcnl_3$, which can be viewed as providing additional constraints (see Section \ref{sec:nonID}). 

\subsection{Identifiability by quadratic loss functionals} \label{sec:ID-RKHS}
We consider the quadratic loss functionals $\errfcnl_1$ and $\errfcnl_4$, and show that they can identify the observation function in the $L^2(\rhoT^L)$-closure of reproducing kernel Hilbert spaces (RKHSs) that are intrinsic to the state model.
\begin{assumption}\label{assumption} We make the following assumptions on the state-space model. 
\begin{itemize}
\item The coefficients in the state model \eqref{eq:SM} satisfy a global Lipschitz condition, and therefore also a linear growth condition: there exists a constant $C>0$ such that $|a(x)-a(y)| + |b(x)-b(y)|\leq C|x-y|$ for all $x,y\in \R$, and $|a(x)|+|b(x)|\leq C(1+|x|)$. Furthermore, we assume that $\inf_{x\in \R}b(x) >0$ for all $x\in \R$.
\item The observation function $\ftrue$ satisfies $\sup_{t\in[0,t_L]}\Ebracket{|\ftrue(X_t)|^2}< \infty$. 
\end{itemize}
\end{assumption}

\begin{theorem}  \label{thm:main}
Given discrete-time data $\{Y_{t_0:t_L}^{(m)} \}_{m=1}^M$ from the state-space model {\rm \eqref{eq:SM}} satisfying Assumption {\rm\ref{assumption}}, let $\errfcnl_1$ and $\errfcnl_4$ be the loss functionals defined in \eqref{eq:errfnl_12corr} and \eqref{eq:errfcnl4}. Denote $p_t(x)$ the density of the state process $X_t$ at time t, and recall that $\rhoT^L$ in \eqref{eq:rho} is the average, in time, of these densities. Let $\L^*$ be the adjoint of the 2nd-order elliptic operator $\L$ in \eqref{eq:Loperator}. Then,  
\begin{itemize}
\item[(a)]  $\errfcnl_1$ has a unique minimizer in $H_1$, the $L^2(\rhoT^L)$ closure of the RKHS $\mH_{K_1}$ with reproducing kernel 
\begin{equation}\label{eq:K1}
K_1 (x,x') = \frac{1}{\rhoT^L(x)\rhoT^L (x')} \sumLavg p_{t_l}(x)p_{t_l}(x'), 
\end{equation}
for $(x,x')$ such that $\rhoT^L(x)\rhoT^L(x')>0$, and $K_1(x,x')=0$ otherwise. 
When the data is continuous ($L\to \infty$),  we have $ K_1(x,x') =  \frac{1}{\rhoT(x)\rhoT(x')} \intTavg p_{t}(x)p_{t}(x')dt$.  
 \item[(b)]  $\errfcnl_4$ has a unique minimizer in $H_4$, the $L^2(\rhoT^L)$ closure of the RKHS $\mH_{K_4}$ with reproducing kernel 
\begin{equation}\label{eq:K4}
K_4 (x,x') = \frac{1}{\rhoT^L(x)\rhoT^L(x')} \sumLavg \L^*p_{t_l}(x)  \L^*p_{t_l}(x'),
\end{equation}
for $(x,x')$ such that $\rhoT^L(x)\rhoT^L(x')>0$, and $K_4(x,x')=0$ otherwise. 
When the data is continuous,  we have $ K_4(x,x') =  \frac{1}{\rhoT(x)\rhoT(x')} \intTavg \L^* p_{t}(x) \L^*p_{t}(x')dt$.  

\item[(c)] $\errfcnl_1+\errfcnl_4$ has a unique minimizer in $H$, the $L^2(\rhoT^L)$ closure of the RKHS $\mH_{K}$ with reproducing kernel 
 \begin{equation}\label{eq:K14}
K (x,x') = \frac{1}{\rhoT^L(x)\rhoT^L(x')} \sumLavg \left[ p_{t_l}(x)p_{t_l}(x') + \L^*p_{t_l}(x)  \L^*p_{t_l}(x') \right],
\end{equation}
for $(x,x')$ such that $\rhoT^L(x)\rhoT^L(x')>0$, and $K(x,x')=0$ otherwise. 
Similarly,  we have $ K(x,x') =  \frac{1}{\rhoT(x)\rhoT(x')} \intTavg [ p_{t}(x)p_{t}(x')+ \L^* p_{t}(x) \L^*p_{t}(x')] dt$ for continuous data. 
 \end{itemize}
In particular,  $\ftrue$ is the unique minimizer of these loss functionals if $\ftrue$ is in $H_1$, $H_4$ or $H$. 
\end{theorem}

To prove this theorem, we first introduce an operator characterization of the RKHS $\mH_{K_1}$ in the next lemma. Similar characterizations hold for the RKHSs $\mH_{K_1}$ and $\mH_{K}$.    
\begin{lemma}\label{lemma:K1}
The function $K_1$ in \eqref{eq:K1} is a Mercer kernel, that is, it is continuous, symmetric and positive semi-definite. Furthermore, $K_1$ is square integrable in $L^2(\rhoT^L\times \rhoT^L)$, and it defines a compact positive integral operator $L_{K_1}: L^2(\rhoT^L)\to L^2(\rhoT^L)$: 
\begin{equation}\label{eq:L_K1}
[L_{K_1} h](x') = \int h(x)K_1(x,x') \rhoT^L(x)dx.
\end{equation}
Also, the RKHS $\mH_{K_1}$ has the operator characterization: $\mH_{K_1} = L^{1/2}_{K_1} (L^2(\rhoT^L))$ and $\{\sqrt{\lambda_i} \psi_i \}_{i=1}^\infty$ is an orthonormal basis of the RKHS $\mH_{K_1}$, where $\{\lambda_i,\psi_i\}$ are the pairs of positive eigenvalues and corresponding eigenfunctions of  $L_{K_1}$. 
\end{lemma}
\begin{proof}
Since the densities of diffusion process are smooth, the kernel $K_1$ is continuous on the support of $\rhoT^L$ and it is symmetric. It is positive semi-definite (see Appendix \ref{sec:appendixA} for a definition) because for any $(c_1,\ldots,c_n)\in \R^n$ and $(x_1,\ldots,x_n)$, we have 
\[
\sum_{i,j=1}^n c_i c_j K(x_i,x_j) =\sumLavg \sum_{i,j=1}^n c_i c_j \frac{ p_{t_l}(x_i)p_{t_l}(x_j)}{\rhoT^L(x_i)\rhoT^L (x_j)}  = \sumLavg \left(\sum_{i=1}^n c_i\frac{p_{t_l}(x_i)}{\rhoT^L(x_i)}  \right)^2 \geq 0. \]
Thus, $K_1$ is a Mercer kernel. 

To show that $K_1$ is square integrable, note first that $p_{t_l}(x) \leq \max_{1\leq k\leq L} p_{t_k}(x) \leq L \rhoT^L(x)$ for any $x$. 
Thus  for each $x,x'$, we have 
\[
 \sumLavg p_{t_l}(x)p_{t_l}(x')\leq L  \rhoT^L(x)\rhoT^L(x') 
 \] 
 and $K_1(x,x')\leq L$. 
It follows that $K_1$ is  in $L^2(\rhoT^L\times \rhoT^L)$. 

Since $K_1$ is positive definite and square integrable, the integral operator $L_{K_1}$ is compact and positive. The operator characterization follows from Theorem \ref{thm:RKHS}. 
\end{proof}

\begin{remark} 
The above lemma is only applicable to discrete-time observations because it uses the bounds $K_1(x,x')\leq L$. When the data is continuous in time on $[0,T]$, we have $K_1\in L^2(\rhoT\times \rhoT)$ if the support of $\rhoT$ is compact. In fact, to show that $K_1$ is square integrable when $\mathrm{supp}(\rhoT)$ is compact, we note  that the probability densities are uniformly bounded above, that is, $p_t(x) \leq  \max_{y\in \R,s\in [0,T]}p_s(y)$. 
Thus  for each $x,x'$, we have 
\begin{align*}
 \intTavg p_{t}(x)p_{t}(x')dt & \leq  \left| \intTavg p_{t}(x)^2\ dt \right|^{1/2} \left| \intTavg p_{t}(x')^2\ dt \right|^{1/2} \\
& = \rhoT(x)^{1/2}\rhoT(x')^{1/2}  \max_{y\in \R,s\in [0,T]}p_s(y)
\end{align*}
 by Cauchy-Schwartz for the first inequality.  Then,
 \[K_1(x,x') = \frac{1}{\rhoT(x)\rhoT (x')}  \intTavg p_{t}(x)p_{t}(x')dt \leq  \rhoT(x)^{-1/2}\rhoT(x')^{-1/2} \max_{y\in \R,s\in [0,T]}p_s(y). \]
It follows that $K_1$ is  in $L^2(\rhoT\times \rhoT)$: 
\[ \int \int K_1^2(x,x') \rhoT(x)\rhoT(x')dxdx' \leq |\mathrm{supp}(\rhoT)| \max_{y\in \R,s\in [0,T]}p_s(y)^2 <\infty. \]
When $\rhoT$ has non-compact support, it remains to be proved that $K_1\in L^2(\rhoT\times \rhoT)$. 
\end{remark}

\begin{proof}[Proof of Theorem \ref{thm:main}] The proof for (a)--(c) are similar, so we focus on (a) and only sketch the proof for (b)--(c).

To prove (a), we only need to show the uniqueness of the minimizer, because Lemma \ref{lemma:K1} has shown that $K_1$ is a Mercer kernel. Furthermore, note that by Lemma \ref{lemma:K1}, the $L^2(\rhoT^L)$ closure of the RKHS $\mH_{K_1}$ is $H_1 =\overline {\mathrm{span}\{\psi_i\}_{i=1}^\infty}$, the closure in $L^2(\rhoT^L)$ of the eigenspace of $L_{K_1}$ with non-zero eigenvalues, where $L_{K_1}$ is the operator defined in \eqref{eq:L_K1}.

For any $f\in H_1$, with the notation $h=f-\ftrue$, we have $\E[f(X_{t})] - \E[Y_{t}] = \E[h(X_t)]$ for each $t$ (recall that $Y_t = \ftrue(X_t)$). Hence, we can write the loss functional as 
\begin{equation}\label{eq:E1_rkhs}
\begin{aligned}
\errfcnl_1(f)  = & \sumLavg \big|\E[f(X_{t_l})] - \E[Y_{t_l}]|^2 =  \sumLavg \big|\E[h(X_{t_l})]|^2 =   \int \int h(x)h(x') \sumLavg p_{t_l}(x)p_{t_l}(x') dxdx' \\
= &  \int \int h(x)h(x') K_1(x,x') \rhoT^L(x)\rhoT^L(x')dxdx' \geq 0.
\end{aligned}
\end{equation}
Thus, $\errfcnl_1$ attains its unique minimizer in $H_1$ at $\ftrue$ iff $\errfcnl_1(\ftrue +h)=0$ with $h\in H_{1}$ 
implies that $h=0$. 
Note that the second equality in \eqref{eq:E1_rkhs} implies that $\errfcnl_1(\ftrue +h) =0$ iff $\E[h(X_{t_l})]=0$, i.e. $\int h(x) p_{t_l}(x) dx =0 $, for all $t_l$. Then, $\int h(x) p_{t_l}(x) \frac{p_{t_l}(x')}{\rhoT^L(x')} dx =0 $ 
for each $t_l$ and $x'$. Thus, the sum of them is also zero: 
\[
0
=  \int h(x) \sumLavg \frac{p_{t_l}(x)  p_{t_l}(x')}{\rhoT^L(x')\rhoT^L(x)} \rhoT^L(x)  dx 
= \int h(x)K_1(x,x') \rhoT^L(x)dx
\]
for each $x'$. By the definition of the operator $L_{K_1}$, this implies that $L_{K_1} h = 0$. Thus, $h=0$ because $h\in H_1$. 

The above arguments hold true when the kernel $K_1$ is from continuous-time data: one only has to replace $\sumLavg$ by the averaged integral in time. This completes the proof for (a). 

The proof of (b) and (c) are the same as above except the appearance of the operator $\L^*$. Note that $\errfcnl_4$ in \eqref{eq:errfcnl4} reads
$ \errfcnl_4(f)  =  \sumLavg \left|\Ebracket{\L f(X_{t_l})}  - \Ebracket{\Delta Y_{t_l}} \right|^2$, thus, it differs from $\errfcnl_1$ only at the expectation $\Ebracket{\L f(X_{t_l})}$.  By integration by parts, we have 
\[
\Ebracket{\L f(X_s) } = \int \L f(x)p_s(x)dx = \int f(x)\L^*p_s(x) dx
\]
for any $f\in C_b^2$. Then, the rest of the proof for Part (b) follows exactly as above with $K_1$ and $L_{K_1}$ replaced by $K_4$ and $L_{K_4}$.   
\end{proof}

\bigskip
The following remarks highlight the implications of the above theorem. We consider only $\errfcnl_1$, but all the remarks apply also to $\errfcnl_4$ and $\errfcnl_1+ \errfcnl_4$. 

\begin{remark}[An operator view of identifiability]\label{rmk:ill-posed}
The unique minimizer of $\errfcnl_1 $ in $H_1$ defined in Theorem {\rm \ref{thm:main}} is the zero of its Frech\'et derivative: $\widehat f = L_{K_1}^{-1} L_{K_1}\ftrue$, which is $\ftrue$ if $\ftrue\in H_1$. In fact, note that with the integral operator $L_{K_1}$, we can write the loss functional $\errfcnl_1$ as 
\[
\errfcnl_1(f) = \innerp{f-\ftrue,L_{K_1}(f-\ftrue)}_{L^2(\rhoT^L)}. 
\]
Thus, the Frech\'et derivative of $\errfcnl_1$ over $L^2(\rhoT^L)$ is $\nabla \errfcnl_1(f) =L_{K_1}(f-\ftrue)$ and we obtain the unique minimizer. Furthermore, this operator representation of the minimizer conveys two important messages about the inverse problem of finding the minimizer of $\errfcnl_1$: (1) it is \emph{ill-defined} beyond $H_1$. In particularly, it is ill-defined on $L^2(\rhoT^L)$ when $L_{K_1}$ is not strictly positive; (2) the inverse problem is ill-posed on $H_1$, because the operator $L_{K_1}$ is compact and its inverse $L_{K_1}^{-1}$ is unbounded. 
 \end{remark}

\begin{remark}[Identifiability and normal matrix in regression]  Suppose $\mH_n =\mathrm{span}\{\phi_i\}_{i=1}^n$ and denote $f =\sum_{i=1}^n c_i  \phi_i$ with  $\phi_i$ being basis functions such as B-splines. As shown in \eqref{eq:errfcnl1}-\eqref{eq:Ab1}, the loss functional $  \errfcnl_1$ becomes a quadratic function with normal matrix $\overline{A}_{1} =\sumLavg A_{1,l}$ with $A_{1,l} = \mathbf{u}_l^\top \mathbf{u}_l$, where $\mathbf{u}_l = (\Ebracket{\phi_1(X_{t_l})}, \ldots, \Ebracket{\phi_n(X_{t_l})}) \in \R^n$. Thus, the rank of the matrix $\overline{A}_{1}$ is no larger than $\min\{n,L\}$. Note that $\overline{A}_{1}$ is the matrix approximation of $L_{K_1}$ on the basis $\{\phi_i\}_{i=1}^n$ in the sense that 
\[
\overline{A}_{1}(i,j) = \innerp{L_{K_1}\phi_i,\phi_j}_{L^2(\rhoT^L)},
\]
for each $1\leq i,j\leq n$. 
Thus, the minimum eigenvalue of $\overline{A}_{1}$ approximates the minimal eigenvalue of $L_{K_1}$ restricted in $\mH_n$. In particular, if $\mH_n$ contains a nonzero element in the null space of $L_{K_1}$, then the normal matrix will be singular; if $\mH_n$ is a subspace of the $L^2(\rhoT^L)$ closure of $\mH_{K_1}$, then the normal matrix is invertible and we can find a unique minimizer. 
\end{remark}



\begin{remark}[Convergence of estimator]\label{rmk:convergence} 
For a fixed hypothesis space, the estimator converges to the projection of $\ftrue$ in $\mH\cap H_1$ as the data size $M$ increases, at the order $O(M^{-1/2})$, with the error coming from the Monte Carlo estimation of the moments of observations. With data adaptive hypothesis spaces, we are short of proving the minimax rate of convergence as in classical nonparametric regression. This is because of the lack of a coercivity condition {\rm \cite{LZTM19pnas,LLMTZ21}}, since the compact operator $L_{K_1}$'s eigenvalue converges to zero. A minimax rate would require an estimate on the spectrum decay of $L_{K_1}$, and we leave this for future research. 
\end{remark}

\begin{remark}[Regularization using the RKHS]
The RKHS $H_{K_1}$ can be further utilized to provide a regularization norm in the Tikhonov regularization (see {\rm \cite{LLA22}}). It has the advantage of being data adaptive and constrains the learning to take place in the function space of learning.  \end{remark}

\paragraph{Examples of the RKHS.} We emphasize that the reproducing kernel and the RKHS are intrinsic to the state model (including the initial distribution). We demonstrate the kernels by analytically computing them when the process $(X_t)$ is either the Brownian motion or the Ornstein-Uhlenbeck (OU) process. For simplicity, we consider continuous-time data. Recall that when the diffusion coefficient in the state-model \eqref{eq:SM} is a constant, the second-order elliptic operators $\L$ is $\L f =\nabla f \cdot a +\frac{1}{2}b^2 \Delta f$ and its joint operator $\L^*$ is 
\[
 \L^*p_s =- \nabla\cdot (a p_s) +\frac{1}{2}b^2 \Delta p_s, 
\]
where $p_s$ denotes the probability density of $X_s$.

\begin{myexample}[1D Brownian motion]
Let $a=0$ and $b=1$. Assume $p_0(x)= \delta_{x_0}$, i.e., $X_0=x_0$. Then, $X_t$ is the Brownian motion starting from $x_0$ and $p_t(x) =\frac{1}{\sqrt{2\pi t} }e^{-\frac{(x-x_0)^2}{2t}}$ for each $t>0$. We have $\rhoT(x) = \intTavg p_t(x)dt = \frac{x-x_0}{T\sqrt{\pi}} \Gamma(-\frac{1}{2}, \frac{(x-x_0)^2}{2T})$ and 
\[
K_1(x,x') = \frac{1}{\rhoT(x)\rhoT(x')} \intTavg  p_s(x)p_s(x')ds = \frac{T\Gamma(0, \frac{(x-x_0)^2+(x'-x_0)^2}{2T})}{2(x-x_0)(x'-x_0)\Gamma(-\frac{1}{2}, \frac{(x-x_0)^2}{2T})\Gamma(-\frac{1}{2}, \frac{(x'-x_0)^2}{2T})}, 
\]
where $\Gamma(s,x) := \int_x^\infty t^{s-1}e^{-t}dt$ is the upper incomplete Gamma function. Also, we have 
\[ \L^*p_s(x) = \phi_2(s,x)p_s(x), \text{ with } \phi_2(s,x) = \left(\frac{1}{s^2}(x-x_0)^2- \frac{1}{s}  \right).  \]
Thus, the reproducing kernel $K_4$ in \eqref{eq:K4} and $K$ in \eqref{eq:K14} from continuous-time data are
\begin{align*}
K_4(x,x') &= \frac{1}{\rhoT(x)\rhoT(x')}  \intTavg \phi_2(s,x)\phi_2(s,x')  p_s(x)p_s(x')ds;  \\
K(x,x') & = \frac{1}{\rhoT(x)\rhoT(x')}  \intTavg (1 + \phi_2(s,x)\phi_2(s,x') ) p_s(x)p_s(x')ds. 
\end{align*}
\end{myexample}

\begin{myexample}[Ornstein-Uhlenbeck process]   Let $a(x)=\theta x$ and $b=1$ with $\theta>0$. Assume $p_0(x)= \delta_{x_0}$, i.e., $X_0=x_0$.  Then,  $X_t = e^{-\theta t}x_0 + \int_{0}^t e^{-\theta (t-s)}dB_s$. It has a distribution $\mathcal{N}(e^{-\theta t}x_0, \frac{1}{2\theta}(1-e^{-2\theta t}) )$, thus $p_t(x) =\frac{1}{\sqrt{2\pi} \sigma_t }\exp(-\frac{(x-x_{0}^t)^2}{2\sigma_t^2})$, where $x_{0}^t:=e^{-\theta t}x_0$ and $\sigma_t^2 :=\frac{1}{2\theta}(1-e^{-2\theta t}) $. Computing the spatial derivatives, we have 
$
 \L^*p_s(x) = \frac{1}{2}\left[ \frac{(x-x_0^s)^2}{\sigma_s^4} - \frac{1}{\sigma_s^2} \right] p_s(x) - (\theta x p_s(x))' = \phi_2(s,x) p_s(x),
 $
 where 
 \[\phi_2(s,x):=  \left[ \frac{(x-x_0)^2}{2\sigma_s^4}  - \frac{1}{2\sigma_s^2}- \theta +\frac{\theta}{\sigma_s^2} x(x-x_0^s) \right].  \]
The reproducing kernels $K_1$ in \eqref{eq:K1}, $K_4$ in \eqref{eq:K4} and $K$ in \eqref{eq:K14} are
\begin{equation*}
\begin{aligned}
K_1(x,x') &= \frac{1}{\rhoT(x)\rhoT(x')}  \intTavg  p_s(x)p_s(x')ds; \\
K_4(x,x') &= \frac{1}{\rhoT(x)\rhoT(x')}  \intTavg \phi_2(s,x)\phi_2(s,x')  p_s(x)p_s(x')ds; \\
K(x,x') & = \frac{1}{\rhoT(x)\rhoT(x')}  \intTavg (1 + \phi_2(s,x)\phi_2(s,x') ) p_s(x)p_s(x')ds. 
\end{aligned}
\end{equation*}
In particular, when the process is stationary, we have  $K_1(x,x')\equiv 1$ and $K_4(x,x') =0 $ because $\L^* p_s = 0$ when $p_s(x) = \frac{2\theta}{\sqrt{2\pi}} \exp(-\theta x^2)$ is the stationary density.  
\end{myexample}

\subsection{Non-identifiability due to stationarity and symmetry}\label{sec:nonID}
When the hypothesis space $\mH$ has a dimension larger than the RKHS's, the quadratic loss functional $\errfcnl_1$ may have multiple minimizers. The constraints of upper and lower bounds, as well as the loss functionals $\errfcnl_2$ and $\errfcnl_3$, can help to identify the observation function. However, as we show next, identifiability may still not hold due to symmetry and/or stationarity. 

\paragraph{Stationary processes} 
When the process $(X_t)$ is stationary,  we have limited information from the moments in our loss functionals. We have $\errfcnl_1(f)=   \left | \Ebracket{Y_{t_1} }- \Ebracket{f(X_{t_1}) } \right|^2 $
with $K_1(x,x') \equiv 1$, so $\errfcnl_1$ can only identify a constant function. Also, the loss functional $\errfcnl_4=0$ because
\[
 \L^*p_s =\partial_s p_s =0; \Leftrightarrow \E[ \L h(X_s)] =0 \text{ for any } h \in C^2_b. 
\] 
In other words, the function space of identifiability by $\errfcnl_1+\errfcnl_4$ is the space of constant functions. Meanwhile, the quartic loss functionals $\errfcnl_2$ and $\errfcnl_3$ also provide limited information: they become  $\errfcnl_2=\left|\E[f(X_{t_1})^2] - \E[Y_{t_1}^2]\right|^2$ and $\errfcnl_3=\left|\E[f(X_{t_2})f(X_{t_{1}})] - \E[Y_{t_2}Y_{t_{1}}]\right|^2$,  the second-order moment and the temporal correlation at one-time instance. 

To see the limitations, consider the finite-dimensional hypothesis space $\mH$ 
in \eqref{eq:mH_with_bds}. As in \eqref{eq:errfcnl-M}, with $f=\sum_{i=1}^n c_i\phi_i$, the loss functional becomes 
\begin{equation*}
\begin{aligned}
\errfcnl(f)  
= &  c^\top \overline{A}_{1} c - 2c^\top \overline{b}_{1}^M + |\E[Y_{t_1}]|^2+ \sum_{k=2}^3 \left|  c^\top A_{k,1} c - b_{k,1}^M \right|^2 ,
\end{aligned}
\end{equation*}
where $\overline{A}_{1} $ is a rank-one matrix, and $  \sum_{k=2}^3  \left|  c^\top A_{k,1} c - b_{k,1}^M \right|^2$ only bring in two additional constraints. Thus, $\errfcnl$ has multiple minimizers in a linear space with dimension greater than 3. One has to resort to the upper and lower bounds in \eqref{eq:mH_with_bds} for additional constraints, which lead to minimizers on the boundary of the resulted convex sets.


\paragraph{Symmetry} When the distribution of the state process  $X_t$ is symmetric, a moment-based loss functional does not distinguish the true observation function from its symmetric counterpart. More specifically,  if a transformation $R:\R\to\R$ preserves the distribution, i.e., $(X_t, t\geq 0)$ and $( R(X_t), t\geq 0)$ have the same distribution, then $\E[f(X_t)] = \Ebracket{f\circ R(X_t)}$ and $\E[f(X_t)f(X_s)] = \Ebracket{f\circ R(X_t)f\circ R(X_s)}$.  Thus, our loss functional will not distinguish $f$ from $f\circ R$. However, this is totally reasonable: the two functions yield the same observation process (in terms of the distribution), thus the observation data does not provide the necessary information for identifying $f$ from $f\circ R$. 

\begin{myexample}[Brownian motion]   
Consider the standard Brownian motion $X_t$, whose distribution is symmetric about $x = 0$ (because the two processes $(X_t,t\geq 0)$ and $(-X_t,t\geq 0)$ have the same distribution). Let the transformation $R$ be $R(x) = -x$. Then, the two functions $f(x)$ and $f(-x)$ lead to the same observation process, thus they cannot be distinguished from the observations. 
\end{myexample}


\section{Numerical Examples}\label{sec4}
We demonstrate the effectiveness and limitations of our algorithm using synthetic data in representative examples.  
The algorithm works well when the state-model's densities vary appreciably in time to yield a function space of identifiability whose distance to the true observation function is small. In this case, our algorithm leads to a convergent estimator as the sample size increases.
We also demonstrate that when the state process (i.e., the Ornstein-Uhlenbeck process) is stationary or symmetric in distribution (i.e., the Brownian motion), the loss functional can have multiple minimizers in the hypothesis space, preventing us from identifying the observation functions (see Section \ref{sec:nonID_num}).

\subsection{Numerical settings}\label{sec:num-settings}
We first introduce the numerical settings used in the tests.
\paragraph{Data generation.} The synthetic data $\{Y_{t_0:t_L}^{(m)} \}_{m=1}^M$ with $t_l= l\Delta t$ are generated from the state model, which is solved by the Euler-Maruyama scheme with a time-step $\Delta t = 0.01$ for $L = 100$ steps. We will consider sample sizes $M \in \{ \floor{10^{3.5+j\Delta}}: j=0,1,2,3,4,\,\, \Delta = 0.0625\}$ to test the convergence of the estimator.  

To estimate the moments in the $A$-matrices and $b$-vectors in  \eqref{eq:Ab1}--\eqref{eq:errfcnl2corr} by Monte Carlo, we generate a new set of independent trajectories $\{X_{t_l}^{(m)}\}_{i=1}^{M'}$ with $M'=10^6$.  
We emphasize that these $X$ samples are independent of the data  $\{Y_{t_0:t_L}^{(m)} \}_{m=1}^M$. 

\paragraph{Inference algorithm.} We follow Algorithm \ref{alg:main} to search for the global minimum of the loss functionals in \eqref{eq:errfcnl-M}. The weights for the $\mathcal{E}_k$'s are $L\sqrt{M}/\| m_k^Y\|$,  where $\|\cdot\|$ is the Euclidean norm on $\R^L$ and 
\[
m_k^Y(l) = \frac{1}{M}\sum_{m=1}^M (Y_{t_l}^{(m)})^k \text{ for } k =1,2  \quad 
\text{ and } \quad 
m_3^Y(l) = \frac{1}{M}\sum_{m=1}^M Y_{t_l}^{(m)} Y_{t_{l+1}}^{(m)}, 
\]
for $l=0,1,\cdots,L-1$.

For each example, we test B-spline hypothesis spaces $\mH$ with dimension in the range $[1,N]$, which is selected by Algorithm \ref{alg:DimensionRange} with degrees in $ \{0,1,2,3\}$. We select the optimal dimension and degree with the minimal 2-Wasserstein distance between the predicted and true distribution of $Y$. The details are presented in Section \ref{sec:dimSelection}.

\paragraph{Results assessment and presentation.} We present three aspects of the estimator $\widehat f$:  
\begin{itemize}[leftmargin=*]\setlength\itemsep{-0.2mm}
\item \textbf{Estimated and true functions.} We compare the estimator with the true function $\ftrue$, along with the $L^2(\rhoT^L)$ projection of $\ftrue$ to the linear space expanded by the elements of $\mH$. 

\item \textbf{2-Wasserstein distance.} We present the 2-Wasserstein distance (see \eqref{eq:W2}) between the distributions of $Y_{t_l} =\ftrue(X_{t_l})$ and $\widehat f(X_{t_l})$ for each time with training data and a new set of randomly generated data. 

\item \textbf{Convergence of $L^2(\rhoT^L)$ error.} We test the convergence of the estimator in $L^2(\rhoT^L)$ as the sample size $M$ increases. The $L^2(\rhoT^L)$ error is computed by the Riemann sum approximation. We present the mean and standard deviation of $L^2(\rhoT^L)$ errors from 20 independent simulations. The convergence rate is also highlighted, and we compare it with the minimax convergence rate in classical nonparametric regression (see e.g., \cite{Gyorfi06,LZTM19pnas}), which is $\frac{s}{2s+1}$ with $s-1$ being the degree of the B-spline basis. This minimax rate is not available yet for our method, see  Remark \ref{rmk:convergence}. 
\end{itemize}

\begin{figure}[h!]
    \centering
         \includegraphics[width=1\textwidth]{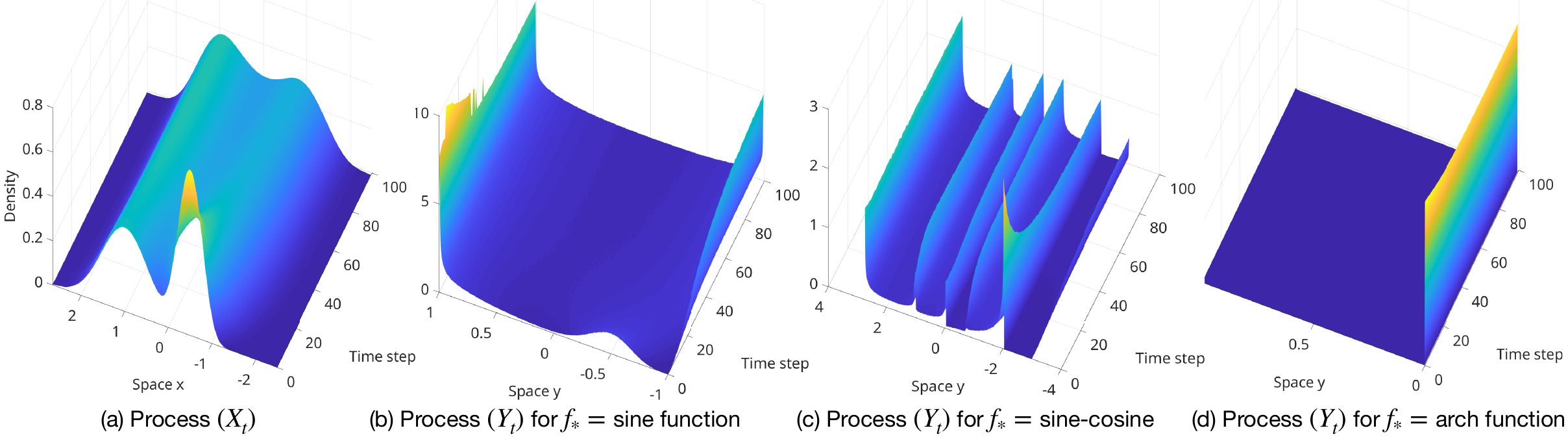}
        \caption{\small Empirical densities from the data trajectories of the state-model process $(X_{t_l})$  in \eqref{eq:DW} and the observation processes $(Y_{t_l})$ with $\ftrue = f_i$, where $f_i$'s are three observation functions in \eqref{eq:obsFns}. Since we do not have data pairs between $(X_{t_l}^{(m)},Y_{t_l}^{(m)})$, these empirical densities are the available information from data. Our goal is to find the function $f$ in the operator that maps the densities of $\{X_{t_l}\}$ to the densities of $\{Y_{t_l}\}$.           
         }
        \label{fig:DW-densities}
\end{figure}

\subsection{Examples} \label{sec:goodExamples}
The state model we consider is a stochastic differential equation with the double-well potential 
\begin{align}
 dX_t & =  (X_t -X_t^3)dt + dB_t, X_{t_0}\sim p_{t_0} \label{eq:DW} 
 \end{align}
where the density of $X_{t_0}$ is the average of $\mathcal{N}(-0.5, 0.2)$ and $\mathcal{N}(1,0.5)$. The  distribution of $X_{t_0:t_L}$ is non-symmetric and far from stationary (see Figure \ref{fig:DW-densities}(a)). Thus the quadratic loss functional $\errfcnl_1$ provides a rich RKHS space for learning.

We consider three observation functions $f(x)$ representing typical challenges: nearly invertible, non-invertible, and non-invertible discontinuous, in the set $\mathrm{supp} (\rhoT)$: 
\begin{equation}\label{eq:obsFns}
\begin{aligned}
&  \text{Sine function: }& \quad f_1(x) =& \sin(x); \\ 
&  \text{Sine-Cosine function: }& \quad f_2(x) =&2\sin(x) + \cos(6x); \\
&\text{Arch function: } & \quad f_3(x) =& \left( - 2(1-x)^3 + 1.5 (1-x) + 0.5\right) \mathbf{1}_{x\in [0,1]}.   
\end{aligned}
\end{equation}
These functions are shown in   \ref{fig:DW-sine}(a) --\ref{fig:DW-Arch}(a). 
They lead to observation processes with dramatically different distributions, as shown in Fig.\ref{fig:DW-densities}(b-d).

\begin{figure}[H]
    \centering
 \includegraphics[width=0.95\textwidth]{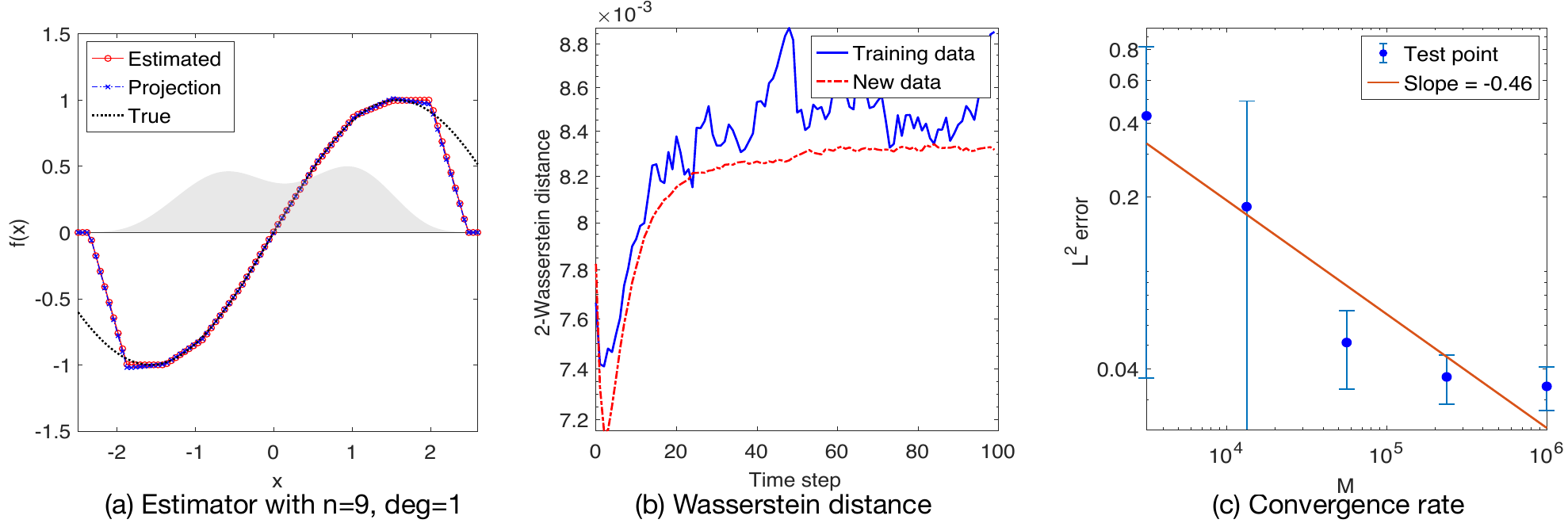}
        \caption{ \small Learning results of Sine function $f_1(x) = \sin(x)$   with model (\ref{eq:DW}).          }
        \label{fig:DW-sine}
 \end{figure}
 \begin{figure}[h!]
     \centering
  \includegraphics[width=0.95\textwidth]{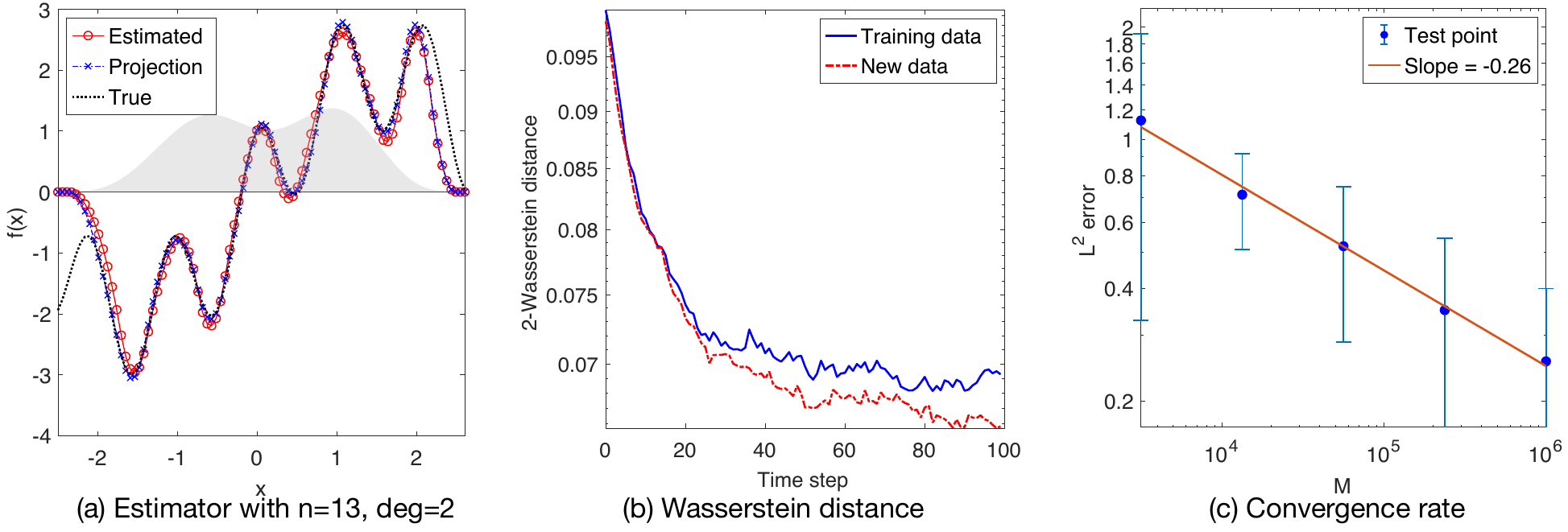}
       \caption{ \small Learning results of Sine-Cosine function $f_2(x) = 2\sin(x) + \cos(6x)$  with model (\ref{eq:DW}).  }
        \label{fig:DW-S1C6}
 \end{figure}
 \begin{figure}[h!]
     \centering
  \includegraphics[width=0.95\textwidth]{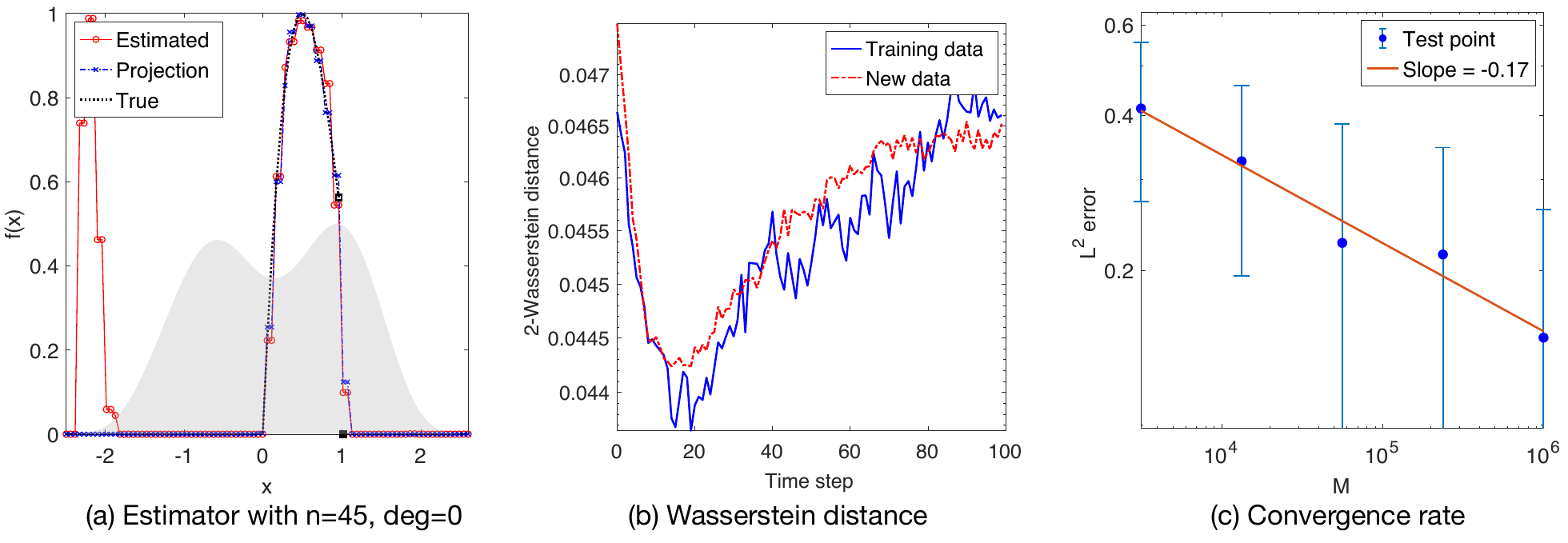}
        \caption{ {\small Learning results of Arch function $f_3$ with model (\ref{eq:DW}).} } 
        \label{fig:DW-Arch}  
\end{figure}

The learning results for these three functions are shown in Figure \ref{fig:DW-sine}--\ref{fig:DW-Arch}. 
For each of these three observation functions, we present the estimator with the optimal hypothesis space, the 2-Wasserstein distance in prediction and the convergence of the estimator in $L^2(\rhoT^L)$ (see Section \ref{sec:num-settings} for details).  

  \textbf{Sine function:}  Fig.~\ref{fig:DW-sine}(a) shows the estimator with degree-1 B-spline basis with dimension $n=9$ for $M=10^6$. 
  The $L^2(\rhoT^L)$ error is $0.0245$ and the relative error is 3.47\%. Fig.~\ref{fig:DW-sine}(b) shows that the Wasserstein distances are small at the scale $10^{-3}$.   Fig.~\ref{fig:DW-sine}(c) shows that the convergence rate of the $L^2(\rhoT^L)$ error is $0.46$.
  This rate is close to the minimax rate $\frac{2}{5}= 0.4$.

\textbf{Sine-Cosine function:}  Fig.~\ref{fig:DW-S1C6}(a) shows the estimator with degree-2 B-spline basis with dimension $n=13$. The $L^2(\rhoT^L)$ error is  $0.1596$ and the relative error is $9.90\%$. Fig.~\ref{fig:DW-S1C6}(b) shows that the Wasserstein distances are at the scale of $10^{-2}$. Fig.~\ref{fig:DW-S1C6}(c) shows that the convergence rate of the $L^2(\rhoT^L)$ error is $0.26$,   
less than the classical minimax rate $\frac{3}{7}\approx 0.42$. Note also that the variance of the $L^2$ error does not decrease as $M$ increases. In comparison with the results for $f_1$ in Fig.\ref{fig:DW-sine}(a), we attribute this relatively low convergence rate and the large variance to the high-frequency component $\cos(6x)$, which is harder to identify from moments than than the low frequency component $\sin(x)$.

\textbf{Arch function:}
 Fig.~\ref{fig:DW-Arch}(a) shows the estimator with degree-0 B-spline basis with dimension $n=45$. The $L^2(\rhoT^L)$ error is  $0.0645$ and the relative error is $14.44\%$. Fig.~\ref{fig:DW-Arch}(b) shows that the Wasserstein distances are small at the scale $10^{-2}$.  Fig.~\ref{fig:DW-Arch}(c) shows that the convergence rate of the $L^2(\rhoT^L)$ error is $ 0.17$,  
 less than the would-be minimax rate $\frac{1}{3}\approx 0.33$.

\textbf{Arch function with observation noise:} To demonstrate that our method can tolerate large observation noise, we present the estimation results from noisy observations of the Arch function, which is the most difficult among the three examples. Suppose that the observation noise $\xi$ in (\ref{eq:obsModel_noisy}) is iid  $\mathcal{N}(0, 0.25)$. Note that the average of $\Ebracket{|Y_t|^2}$  is about $0.2$, so the signal-to-noise ratio is about $\frac{\E[|Y|^2]}{\E[\xi^2]} \approx 0.8$. Thus, we have a relatively large noise. 
However, our method can identify the function using the moments of the noise as discussed in Section \ref{sec:noisyAlg}. 
 Fig.~\ref{fig:DW-Arch-Noise}(a) shows the estimator with degree-1 B-spline basis with dimension $n=24$. The $L^2(\rhoT^L)$ error is  $0.1220$ and the relative error is $27.32\%$. Fig.~\ref{fig:DW-Arch-Noise}(b) shows that the Wasserstein distances are small at the scale $10^{-3}$. The Wasserstein distances is approximated from samples of the noisy data $Y = f_{true}(X) +\xi$ and the noisy prediction $\widehat Y = \widehat f(X) + \xi$.  Fig.~\ref{fig:DW-Arch-Noise}(c) shows that the convergence rate of the $L^2(\rhoT^L)$ error is $ 0.14$. The estimation is not as good as the noise-free case because the noisy observation data lead to milder lower and upper bound restrictions in (\ref{eq:mH_with_bds}).
We emphasize that the \emph{tolerance to noise is exceptional} for such an ill-posed inverse problem, and the key is our use of moments, which averages the noise so that the error occurs at the scale $O(\frac{1}{\sqrt{M}})$.    

 \begin{figure}       
        \centering
  \includegraphics[width=1\textwidth]{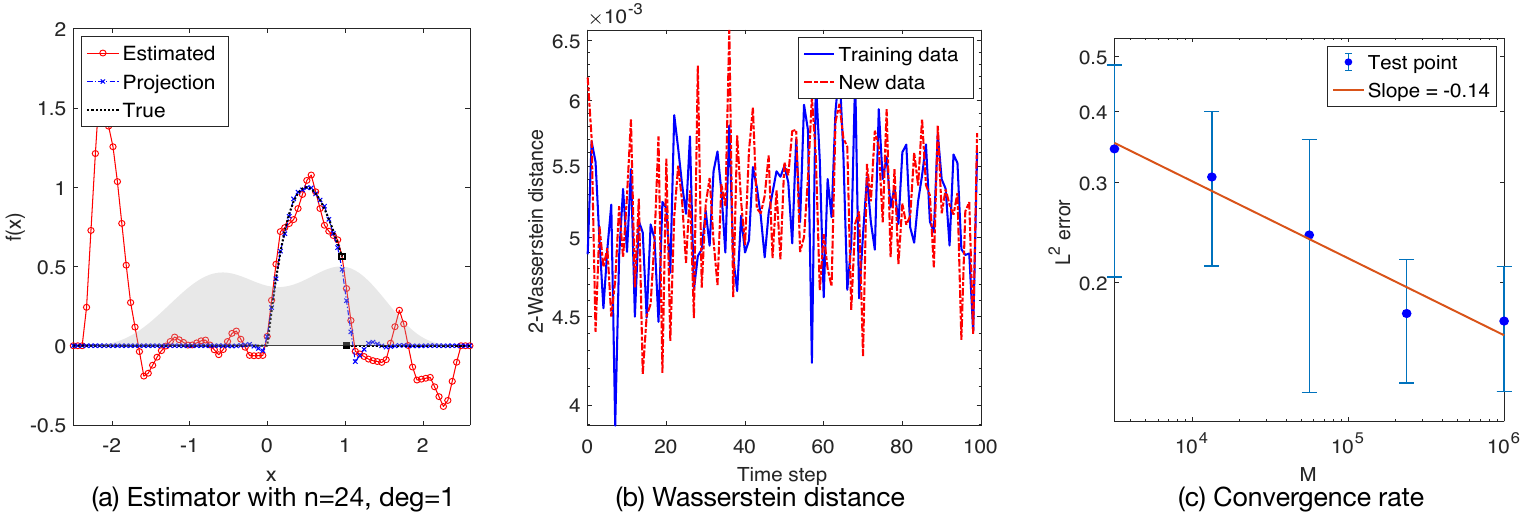}
       \caption{ {\small Learning results of Arch function $f_3$ with model (\ref{eq:DW}) and i.i.d Gaussian observation noise. }}
        \label{fig:DW-Arch-Noise} 
\end{figure}


\bigskip
We have also tested piecewise constant observation functions. Our method has difficulty in identifying such functions, due to two issues: (i) the uniform partition often misses the jump discontinuities (even the projection of $\ftrue$ has a large error); and (ii) the moments we considered depend on the observation function non-locally, thus, they provide limited information to identify the true function from its local perturbations. We leave it for future research to overcome these difficulties by searching the jump discontinuities and by introducing moments detecting local information.

\subsection{Limitations}\label{sec:nonID_num}

We demonstrate by examples the non-identifiability due to symmetry and stationarity. 

\paragraph{Symmetric distribution} Let the state model be the Brownian motion with initial distribution $\text{Unif}(0,1)$. The state process $(X_t)$ has a distribution that is symmetric with respect to the line $x=\frac{1}{2}$, i.e., the processes $(X_t)$ and $(1-X_t)$ have the same distribution. Thus, with the reflection function $R(x) = 1-x$, the processes $f(X_t)$ and  $f\circ R(X_t)$ have the same distribution, and the observation data does not provide information for distinguishing $f$ from $f\circ R$. The loss functional (\ref{eq:errfnl_12corr}) has at least two minima. 

Figure \ref{fig:BM_sin} shows that our algorithm finds the reflection of the true function $\ftrue=\sin(x)$. The hypothesis space $\mathcal{H}$ has B-spline basis functions with degree 2 and dimension 58. Our estimator is close to $\ftrue\circ R (x) = \sin(1-x)$. Its $L^2(\rhoT^L)$ error is $1.1244 $ and its reflection's $L^2(\rhoT^L)$ error is $0.0790 $. Both the estimator and its reflection correctly predict the distribution of the observation process $(Y_t)$.

\begin{figure}[h!]
     \centering
  \includegraphics[width=0.6\textwidth]{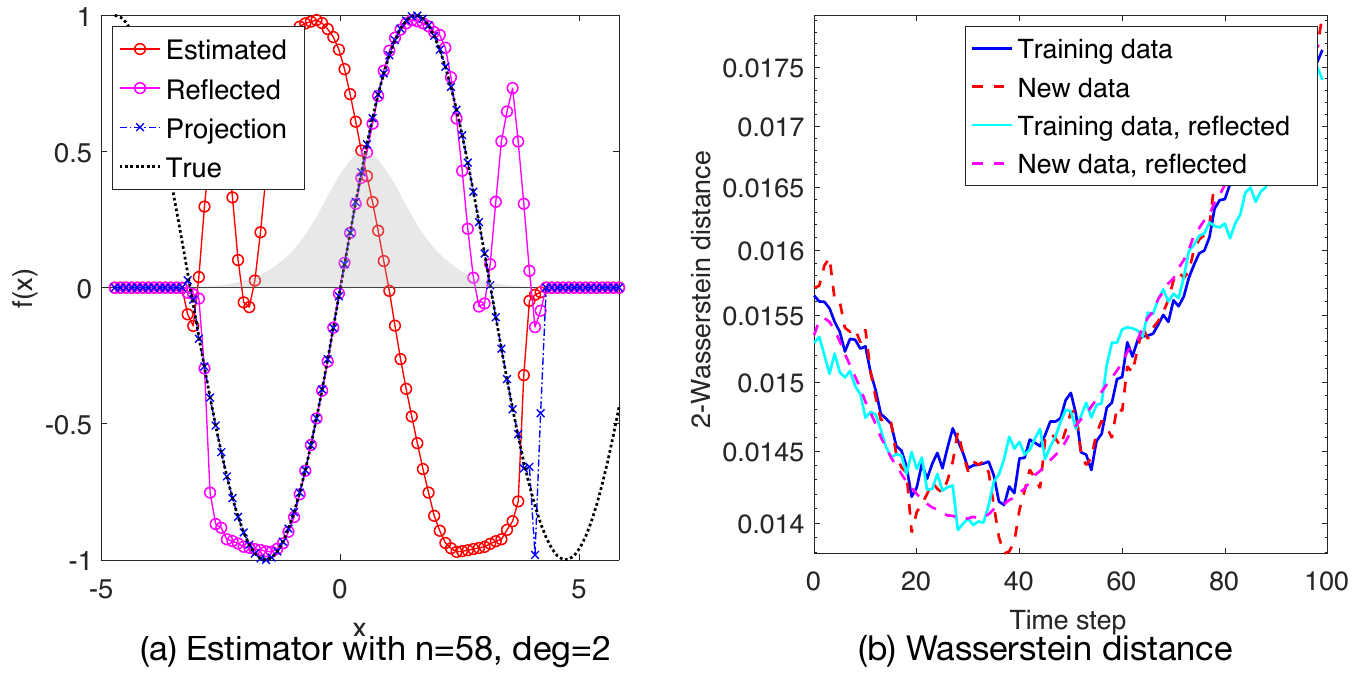}
        \caption{\small  Learning results of $\ftrue(x) = \sin(x)$ with the state model being $X_t= B_t+ X_0$ where $X_{0}\sim \mathrm{Unif}(0,1)$. Due to the symmetry with respect to the line $x=\frac{1}{2}$, the estimator $\widehat{f}(x)$ and its reflection $\widehat{f}(1-x)$ are indistinguishable by the loss functional and they lead to similar prediction of the distribution of $\{Y_{t_l}\}$.  
        }
        \label{fig:BM_sin}
\end{figure}

\bigskip
\paragraph{Stationary process} When the diffusion process $(X_t)$ is stationary, the loss functional $\eqref{eq:errfnl_12corr}$ provides limited information about the observation function. As discussed in Section \ref{sec:nonID}, the matrix $\overline{A}_1$ has rank 1, and $\errfcnl_2=0$ and $\errfcnl_3=0$ lead to only two more constraints. The constraints from the upper and lower bounds in \eqref{eq:mH_with_bds} play a major role in leading to a minimizer at the boundary of the convex set $\mH$. 

Figure \ref{fig:OU_sin} shows the learning results with the stationary Ornstein-Uhlenbeck process $ dX_t = -X_t dt +  dB_t$ and with the observation function $\ftrue(x)=\sin(x)$. The stationary density of $(X_t)$ is $\mathcal{N}(0, \frac{1}{2})$. Due the limited information, the estimator has a large $L^2(\rhoT^L)$ error, which is $0.2656$ and its prediction has large 2-Wasserstein distances oscillating near $0.1290$. 

\begin{figure}[h!]
     \centering
  \includegraphics[width=0.6\textwidth]{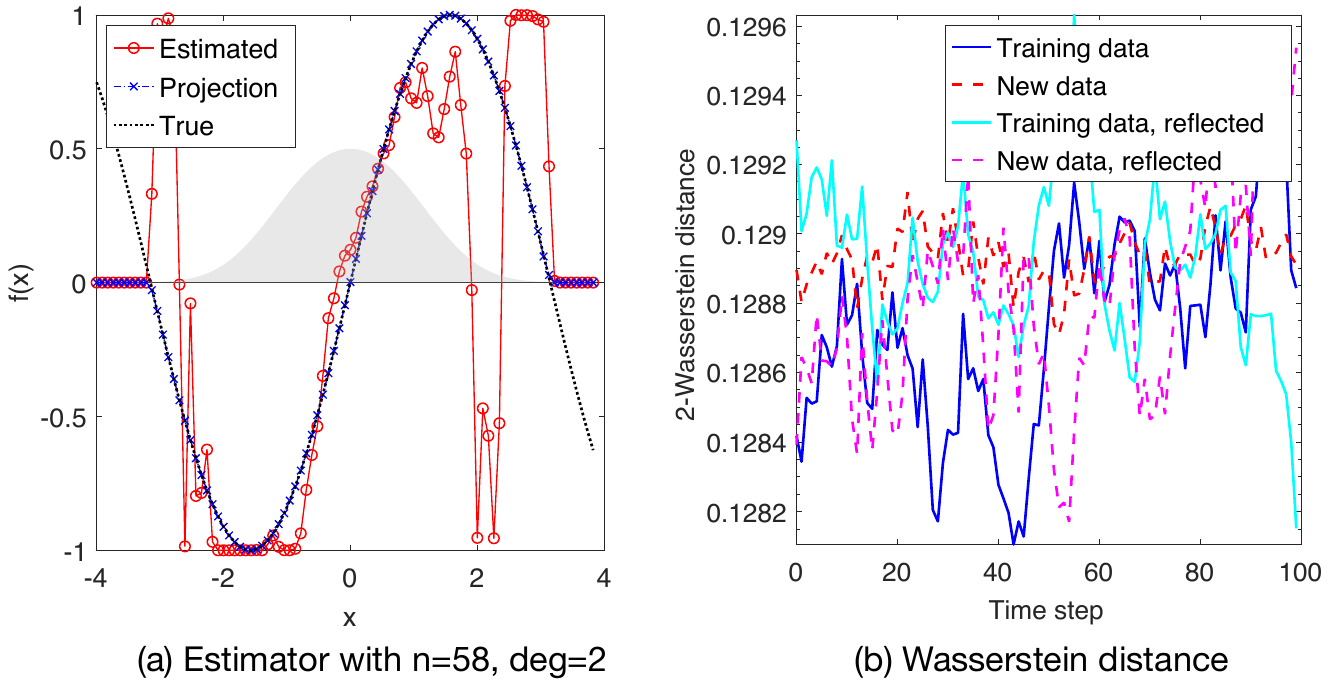}
  \caption{\small Learning results of $\ftrue(x) =\sin(x)$ with stationary Ornstein-Uhlenbeck process. Due to limited information from the moments, the estimator is inaccurate due to its reliance on the upper and lower bound constraints.   
  }
        \label{fig:OU_sin}
        \hfill
\end{figure}

\section{Discussions and conclusion}\label{sec5}
We have proposed a nonparametric learning method to estimate the observation functions in nonlinear state-space models. It matches the generalized moments via constrained regression. The algorithm is suitable for large sets of unlabeled data. Moreover, it can deal with challenging cases when the observation function is non-invertible. We address the fundamental issue of identifiability from first-order moments. We show that the function spaces of identifiability are the closure of RKHS spaces intrinsic to the state model. Numerical examples show that the first two moments and temporal correlations, along with upper and lower bounds, can identify functions ranging from piecewise polynomials to smooth functions and tolerate considerable observation noise. The limitations of this method, such as non-identifiability due to symmetry and stationarity, are also discussed.

This study provides a first step in the unsupervised learning of latent dynamics from abundant unlabeled data. There are several directions calling for further exploration: (i) a mixture of unsupervised and supervised learning that combines unlabeled data with limited labeled data, particularly for high-dimensional functions; (ii) enlarging the function space of learning, either by construction of more first-order generalized moments or by designing experiments to collect more informative data; (iii) joint estimation of the observation function and the state model.

 \appendix
\section{A review of RKHS}\label{sec:appendixA}
\paragraph{Positive definite functions}

We review the definitions and properties of positive definite kernels. The following is a real-variable version of the definition in \cite[p.67]{BCR84}.  

\begin{definition}[Positive definite function]\label{def_spd}
Let $X$ be a nonempty set. A function $G: X\times X\rightarrow \R$ is positive definite if and only if it is symmetric (i.e. $G(x,y)=G(y,x)$) and
$ \sum_{j,k=1}^{n}c_jc_kG(x_j,x_k)\geq 0 $
for all $n\in \mathbb{N}$, $\{x_1,\ldots,x_n\}\subset X$ and $\mathbf{c}=(c_1,\ldots,c_n) \in \R^n$. The function $\phi$ is strictly positive definite if the equality hold only when $\mathbf{c}=\mathbf{0} \in \R^n$. 
\end{definition}

 \begin{theorem}[Properties of positive definite kernels]\label{t52}
 Suppose that $k, k_1, k_2: X \times X \subset\mathbb{R}^d\times\mathbb{R}^d\to \mathbb{R}$ are positive definite kernels. Then
\begin{enumerate} \setlength\itemsep{0mm} 
\item[(a)] $k_1k_2$ is positive definite. (\cite[p.69]{BCR84})
\item[(b)] Inner product $\langle u,v\rangle=\sum_{j=1}^du_jv_j$ is positive definite (\cite[p.73]{BCR84})
\item[(c)] $f(u)f(v)$ is positive definite for any function $f: X \to \mathbb{R}$ (\cite[p.69]{BCR84}).
\end{enumerate}
\end{theorem}

\paragraph{RKHS and positive integral operators}

We review the definitions and properties of the Mercer kernel, the RKHS, and the related integral operator,  see e.g., \cite{cucker2007learning} for them on a compact domain   \cite{Sun03Mercer} for them on a non-compact domain.  

Let $(X,d)$ be a metric space and $G:X\times X\to\R$ be continuous and symmetric. We say that $G$ is a Mercer kernel if it is positive definite (as in Definition \ref{def_spd}). The reproducing kernel Hilbert space (RKHS) $\mH_G$ associated with $G$ is defined to be closure of $\mathrm{span}\{G(x,\cdot):x\in X \}$ with the inner product
\begin{equation*}
\langle f , g\rangle_{\mH_G} =  \sum_{i=1,j=1}^{n, m} c_i d_j G(x_i,y_j)
\end{equation*}
for any $f=\sum_{i=1}^n c_i G(x_i,\cdot)$ and $g=\sum_{j=1}^n d_k G(x_j,\cdot)$. It is the unique Hilbert space such that: (1) the linear space $\mathrm{span}\{G(\cdot,y), y\in X\}$ is dense in it; (2) it has the reproducing kernel property in the sense that for all $f\in \mH_G$ and $x\in X$, $f(x) = \langle G(x,\cdot), f\rangle_{G}$ (see \cite[Theorem 2.9]{cucker2007learning}). 

By means of the Mercer Theorem, we can characterize the RKHS $\mH_G$ through the integral operator associated with the kernel. Let $\mu$ be a nondegenerate Borel measure on $(X,d)$ (that is, $\mu(U)>0$ for every open set $U\subset X$). Define the integral operator $L_G$ on  $L^2(X,\mu)$ by 
\[
L_Gf(x)  =\int_X G(x,y)f(y)d\mu(y). 
\] 
The RKHS has the operator characterization (see e.g., \cite[Section 4.4]{cucker2007learning} and \cite{Sun03Mercer}): 
\begin{theorem} \label{thm:RKHS}
Assume that the $G$ is a Mercer kernel and $G\in L^2(X\times X, \mu\otimes \mu)$. Then 
\begin{enumerate} \setlength\itemsep{0mm} 
\item $L_G$ is a compact positive self-adjoint operator. It has countably many positive eigenvalues $\{\lambda_i\}_{i=1}^\infty$ and corresponding orthonormal eigenfunctions $\{\phi_i\}_{i=1}^\infty$. Note that when zero is an eigenvalue of $L_G$, the linear space  $H=\mathrm{span}\{\phi_i\}_{i=1}^\infty$ is a proper subspace of $L^2(\mu)$. 
\item $\{\sqrt{\lambda_i} \phi_i \}_{i=1}^\infty$ is an orthonormal basis of the RKHS $\mH_G$. 
\item The RKHS is the image of the square root of the integral operator, i.e., $\mH_G=L_G^{1/2} L^2(X,\mu)$. 
\end{enumerate}
\end{theorem}

\section{Algorithm details}
\subsection{B-spline basis and dimension of the hypothesis space}\label{sec:append_HypoSpace}
The choice of hypothesis space is important for the nonparametric regression. We choose the basis functions to be the B-splines. To select an optimal dimension of the hypothesis space, we introduce a new algorithm to estimate the range for the dimension and then we select the optimal dimension that minimizes the 2-Wasserstein distance between the measures of data and prediction. 

\paragraph{B-Spline basis functions} We briefly review the definition of B-spline basis functions and we refer to \cite[Chapter 2]{piegl1997_NURBSBook} and \cite{lyche2018_FoundationsSpline} for details. 
Given a nondecreasing sequence of real numbers, called knots, $(r_0,r_1,\ldots, r_m)$, the B-spline basis functions of degree $p$, denoted by $\{N_{i,p} \}_{i=0}^{m-p-1}$, are defined recursively as
\begin{equation*}
\begin{aligned}
  &N_{i,0}(r) = 
  \left\{
    \begin{array}{lr}
      1,\ &r_i \leq r < r_{i+1}\\
      0,\ &\text{otherwise}
    \end{array}
  \right.,\qquad
  N_{i,p}(r) = \frac{r - r_i}{r_{i + p} - r_i} N_{i, p-1}(r) + \frac{r_{i+p+1} - r}{r_{i + p + 1} - r_{i + 1}}N_{i + 1, p- 1}(r).
\end{aligned}
\end{equation*}
Each function $N_{i,p}$ is a nonnegative local polynomial of degree $p$, supported on $[r_i,r_{i+p+1}]$. At a knot with multiplicity $k$, it is $p-k$ times continuously differentiable. Hence, the differentiability increases with the degree but decreases when the knot multiplicity increases. The basis satisfies a partition unity property: for each $r\in [r_i,r_{i+1}]$, $\sum_{j} N_{j,p}(r)  = \sum_{j=i-p}^i N_{j,p}(r) =1$. 

We set the knots of the spline functions to be a uniform partition of $[R_{min}, R_{max}]$ (the support of the measure $\rhoT^L$ in \eqref{eq:rho})
$R_{min} =  r_{0} \leq r_1 \leq \cdots \leq r_{m} = R_{min}$.
For any choice of degree $p$, we set the basis functions of the hypothesis space $\mH$, contained in a subspace with dimension $n=m-p$, to be 
$$\phi_i(r) = N_{i, p}(r), \ i = 0,\dots, m-p-1. $$ 
Thus, the basis functions $\{\phi_i\}$ are piecewise degree-$p$ polynomials with knots adaptive to $\rhoT^L$. 


\paragraph{Dimension of the hypothesis space.} The choice of dimension $n$ of $\mH$ is important to avoid under- and over-fitting: we choose it by minimizing the 2-Wasserstein distance between the empirical distributions of observed process $(Y_t)$ and that predicted by our estimated observation function.  
To reduce the computational burden, we proceed in 2 steps: first we determine a rough range for $n$, and then within this range we select the dimension with the minimal Wasserstein distance.

Step 1: we introduce an algorithm, called \emph{Cross-validating Estimation of Dimension Range (CEDR)}, to estimate the range $[1,N]$ for the dimension of the hypothesis space, based on the quadratic loss functional $\errfcnl_1$. Its main idea is to restrict $N$ to avoid overly amplifying the estimator's sampling error, which is estimated by splitting the data into two sets. 
It incorporates the function space of identifiability in Section \ref{sec:ID-RKHS} into the SVD analysis \cite{fierro1997_RegularizationTruncated,hansen_LcurveIts_a} of the normal matrix and vector from $\errfcnl_1$. 

The CEDR algorithm estimates the sampling error in the minimizer of loss functional $\errfcnl_1$ through SVD analysis in three steps. First, we compute the normal matrix $\overline{A}_1$ and vector $\overline{b}_1$  in \eqref{eq:Ab1} by Monte Carlo; to estimate the sampling error in $\overline{b}_1$, we compute two copies, $b$ and $b'$, of $\overline{b}_1$ from two halves of the data: 
\begin{equation}\label{eq:bb'}
\begin{aligned}
b(i) =  \sumLavg  \Ebracket{\phi_i(X_{t_l}) } \frac{2}{M}\sum_{m = 1}^{\lfloor\frac{M}{2}\rfloor}{Y_{t_l}^{(m)}},\quad 
 b'(i)=  \sumLavg  \Ebracket{\phi_i(X_{t_l}) } \frac{2}{M}\sum_{m = \lfloor\frac{M}{2}\rfloor+1}^{M}{Y_{t_l}^{(m)}}. 
\end{aligned}
\end{equation}
Second, we implement an eigen-decomposition to find an orthonormal basis of $L^2(\rhoT^L)$, the default function space of learning. We view the matrix $\overline{A}_1$ as a representation of the integral operator $L_{K_1}$ in Lemma \ref{lemma:K1} on $\mH=\mathrm{span}\{\phi_i\}_{i=1}^n$. The eigen-decomposition requires the generalized eigenvalue problem 
\begin{equation}\label{eq:gEigenP}
 \overline{A}_1 u =\lambda  B u, \quad \text{ where }  B = ( \innerp{\phi_i,\phi_j}_{L^2(\rhoT^L)}) 
 \end{equation}
 (see \cite[Theorem 5.1]{LangLu21}). Denote the eigen-pairs by $\{\sigma_i, u_i\}$, where the eigenvalues $\{\sigma_i\}$ are non-increasingly ordered and the eigenvectors are subject to normalization $u_i^\top B u_j= \delta_{i,j}$. 
Thus, we have $\overline{A}_1 = \sum_{i=1}^n \sigma_i u_i u_i^\top $ (assuming that all $\sigma_i$'s are positive; otherwise, we drop those zero eigenvalues). The least-squares estimators from $b$ and $b'$ are $c = \sum_{i=1}^n \frac{u_i^\top b}{\sigma_i} u_i$ and $c' = \sum_{i=1}^n \frac{u_i^\top b'}{\sigma_i} u_i$, respectively. Third, the difference between their function estimators represents the sampling error
 (with $\Delta c = c-c'$)
\begin{equation}\label{eq:err_gk}
\begin{aligned}
g(n):= & \| \widehat f - \widehat f'  \|_{L^2(\rhoT^L)}^2 = \| \sum_{k=1}^n \Delta c_k \phi_k \|_{L^2(\rhoT^L)}^2 = \sum_{i,j=1}^n \Delta c_i \innerp{\phi_i,\phi_j}_{L^2(\rhoT^L)} \Delta c_j= \Delta c^\top B \Delta c \\ 
& =  \sum_{i,j=1}^n \frac{u_i^\top (b-b')}{\sigma_i} u_i^\top B u_j \frac{u_j^\top (b-b')}{\sigma_j} = \sum_{i=1}^n r_i^2, 
\end{aligned}
\end{equation}
where $r_i= \frac{|u_i^\top (b-b')|}{\sigma_i}$. The ratio $r_i$ is in the same spirit as the \emph{Picard projection ratio} $\frac{|u_i^\top b|}{\sigma_i}$ in  \cite{hansen_LcurveIts_a}, which is used to detect overfitting. 
Note that the eigenvalues $\sigma_i$ will vanish as $n$ increases because the operator $L_{K_1}$ is compact.  Clearly, the sampling error $g(n)$ should be less than $\|\ftrue\|_{L^2(\rhoT^L)}^2$, which is the average of the second moments. 
Thus, we set $N$ to be
\begin{equation}\label{eq:N2}
\begin{aligned}
N &= \max\{k\geq 1: g(k) \leq \tau  \}, \text{ where } \tau=  \frac{1}{LM}\sum_{l=1, m=1}^{L,M} |Y_{t_l}^{(m)} |^2.
\end{aligned}
\end{equation}
 We note that this threshold is relatively large, neglecting the rich information in $g$, a subject worthy of further investigation.

Algorithm \ref{alg:DimensionRange} summarizes the above procedure. 
\begin{algorithm}[H]
\caption{ {\small Cross-validating Estimation of Dimension Range (CEDR) for hypothesis space 
}}\label{alg:DimensionRange}
{\small
\begin{algorithmic}[1]
\Require{The state model and data $\{Y_{t_0:t_L}^{(m)} \}_{m=1}^M$. } 
\Ensure{A range $[1,N]$ for the dimension of the hypothesis space for further selection.}
\State  Estimate the empirical density $\rhoT$ in \eqref{eq:rho}  and find its support $[R_{min}, R_{max}]$. 
\State Set $n=1$ and $g(n)=0$. Estimate the threshold $\tau$ in \eqref{eq:N2}.  
\While{$g(n)\leq \tau$}
\State Set $n\leftarrow n+1$. Update the basis functions, Fourier or B-spline, as in Section \ref{sec:basisFn_dim}. 
\State Compute normal matrix $\overline{A}_1$ in \eqref{eq:Ab1} by Monte Carlo. Also, compute  $b$ and $b'$ in \eqref{eq:bb'}. 
\State Eigen-decomposition of $\overline{A}_1$ as in \eqref{eq:gEigenP}; return $\overline{A}_1 =\sum_{i=1}^n u_i \sigma_i u_i^T$ with $u_i^\top B u_j= \delta_{i,j}$.    
\State Compute the Picard projection ratios: $r_i = \frac{|u_i^\top (b-b')|}{\sigma_i}$ for $i=1,\ldots,n$ and $g(n)=  \sum_{i=1}^n r_i^2$.
\EndWhile
\State Return $N=n$. 
\end{algorithmic}
}\end{algorithm}

Step 2:  We select the dimension $n$ and degree for B-spline basis functions to be the one with the smallest 2-Wasserstein distance between the distribution of the data and that of the predictions. 
More precisely, let $\mu^f_{t_l}$ and $\mu^{\widehat f}_{t_l}$ denote the distributions of $Y_{t_l} = f(X_{t_l})$ and $\widehat{f}(X_{t_l})$, respectively. Let $F_{t_l}$ and $\widehat F_{t_l}$ denote their cumulative distribution functions (CDF), with $F_{t_l}^{-1}$ and $\widehat F_{t_l}^{-1}$ being their inversion. We compute $F_{t_l}$  from the data and $\widehat F_{t_l}$ from independent simulation. We approximate their inversions by quantile, and compute the 2-Wasserstein distance  
\begin{align}\label{eq:W2}
\left(\sumLavg W_2(\mu^f_{t_l}, \mu^{\widehat f}_{t_l})^2  \right)^{1/2} \quad \text{ with } W_2(\mu^f_{t_l}, \mu^{\widehat f}_{t_l}) = \left( \int_0^1 (F^{-1}_{t_l}(r) - \widehat F^{-1}_{t_l}(r)  )^2 dr\right)^\frac{1}{2},
\end{align}
This method of computing the Wasserstein distance is based on an observation in \cite{carrillo2004wasserstein}, and it has been used in \cite{panaretos2019statistical,Kolbe2020WassersteinDistance}. 
Recall that the 2-Wasserstein distance $W_2(\mu,\nu)$ of two probability density functions $\mu$ and $\nu$ over $\Omega$ with finite second order moments is given by 
$ W_2(\mu,\nu) = \left( \inf\limits_{\gamma\in\Gamma(\mu,\nu)} \int_{\Omega\times\Omega} |x-y|^2d\gamma(x,y)  \right)^\frac{1}{2}$,
where $\Gamma(\mu,\nu)$ denotes the set of all measures on $\Omega \times \Omega $ with $\mu$ and $\nu$ as marginals. Let $F$ and $G$ be the CDFs of $\mu$ and $\nu$ respectively, and let $F^{-1}$ and $G^{-1}$ be their quantile functions. Then the $L^2$ distance of the quantile functions $ d_2(\mu,\nu) = \left(\int_0^1 |F^{-1}(r)-G^{-1}(r)dr |^2  \right)^{\frac{1}{2}}$ is equal to the 2-Wasserstein distance $W_2(\mu,\nu)$.

\subsection{Optimization with multiple initial conditions}\label{sec:multiIC}
 With the convex hypothesis space in \eqref{eq:mH_with_bds}, the minimization in  \eqref{eq:errfcnl-M} is a constrained optimization problem and it may have multiple local minima. Note that the loss functional $\errfcnl^M(c)$ in \eqref{eq:errfcnl-M} consists of a quadratic term and two quartic terms. The quadratic term, which represents $\errfcnl_1^M$ in \eqref{eq:errfcnl1}, has a Hessian matrix $\overline{A}_{1}$ which is often not full rank because it is the average of rank-one matrices  by its definition  \eqref{eq:Ab1}. Thus, the quadratic term has a valley of minima in the kernel of $\overline{A}_{1}$. The two quartic terms have valleys of minima at the intersections of the ellipse-shaped manifolds $\{c\in \R^n:   c^\top A_{k,l} c = b_{k,l}^M \}_{l=1}^L$ for $k=2,3$.  Also, symmetry in the distribution of the state process will also lead to multiple minima (see Section \ref{sec:nonID} for more discussions). 

  To reduce the possibility of obtaining a local minimum, we search for a minimizer from multiple initial conditions. We consider the following initial conditions: (1) the  least squares estimator for the quadratic term; (2) the minimizer of the quadratic term in the hypothesis space, which is solved by least squares with linear constraints using @MATLAB function \textsf{lsqlin}, starting from the LSE estimator; (3) the minimizers of the quartic terms over the hypothesis space, which is found by constrained optimization through @MATLAB \textsf{fmincon} with the interior-point search. Then, among the minimizers from these initial conditions, we take the one leading to the smallest 2-Wasserstein distance. 

\section{Selection of dimension and degree of the B-spline basis}\label{sec:dimSelection}
We demonstrate the selection of the dimension and degree of the B-spline basis functions of the hypothesis space. 
As described in Section \ref{sec:basisFn_dim}, we select the dimension and degree in two steps: we first select a rough range for the dimension by the Cross-validating Estimation of Dimension Range (CEDR) algorithm; then we pick the dimension and degree to be the ones with minimal 2-Wasserstein distance between the true and estimated distribution of the observation processes.

The CEDR algorithm helps to reduce the computational cost by estimating the dimension range for the hypothesis space. It is based on an SVD analysis of the normal matrix $\overline{A}_1$ and vector $\overline{b}_1$ from the quadratic loss functional $\errfcnl_1$. The key idea is to control the sampling error's effect on the estimator in the metric of the function space of learning. The sampling error is estimated by computing two copies of the normal vector through splitting the data into two halves. The function space of learning plays an important role here: it directs us to use a generalized eigenvalue problem for the SVD analysis. This is different from the classical SVD analysis in \cite{hansen_LcurveIts_a}, where the information of the function space is neglected.

\begin{figure}[h]
\centering
              \includegraphics[width=0.6\textwidth]{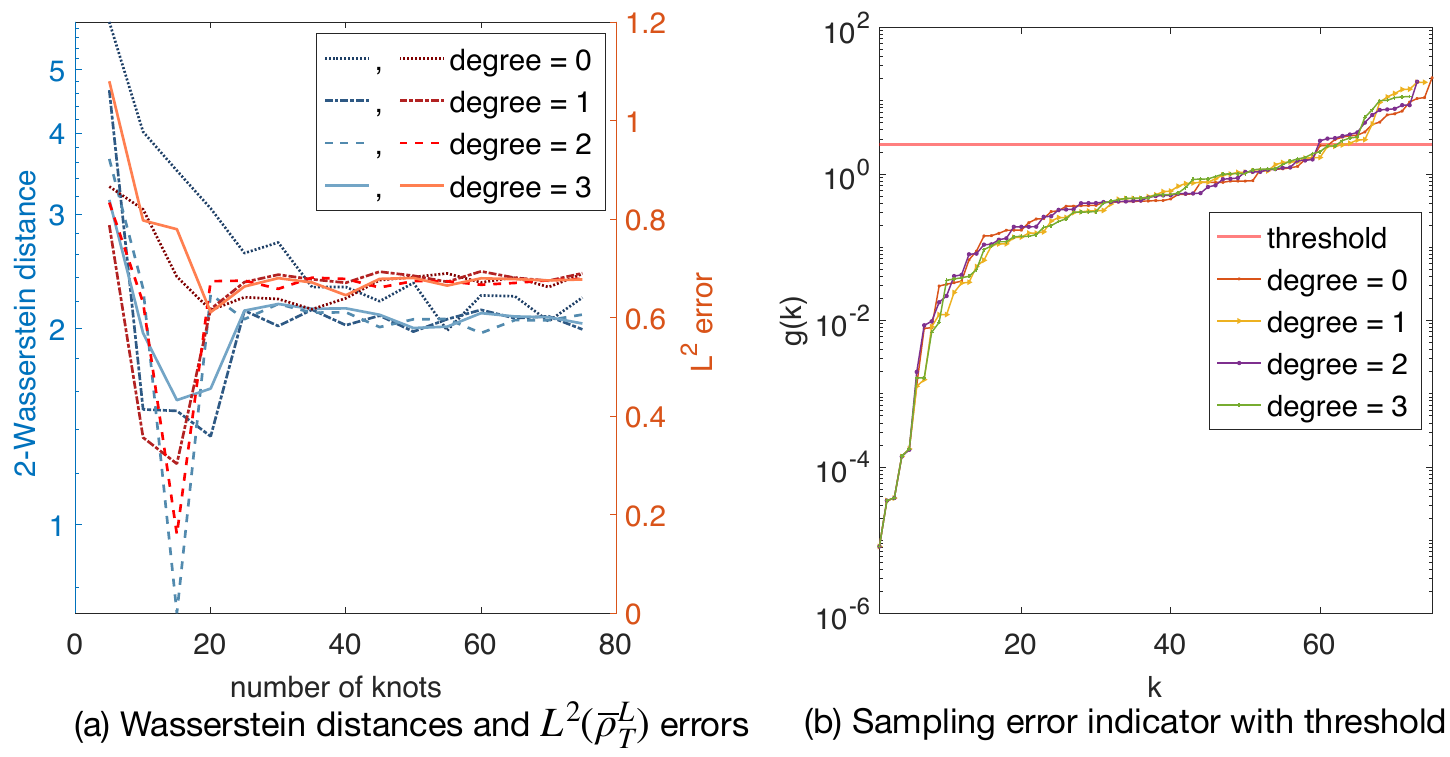}
    \caption{\small The selection of the dimension and the degree of B-spline basis functions in the case of Sine-Cosine function. In (a), the 2-Wasserstein distance reaches minimum among all cases when the degree is 2 and the knot number is 15,  at the same time as the $L^2(\rhoT^L)$ error reaches the minimum. Figure (b) shows the cross-validating error indicator $g$ (defined in \eqref{eq:err_gk}) for selecting the dimension range $N$, suggesting an upper bound $N=60$ with the threshold.} 
     \label{fig:dimSelection}
\end{figure}

Figure \ref{fig:dimSelection} shows the dimension selection by 2-Wasserstein distances and by the CEDR algorithm for the example of sine-cosine function.  To confirm the effectiveness of our CEDR algorithm, we compute the 2-Wasserstein distances for all dimensions in (a), side-by-side with the CEDR sampling error indicator $g$ in (b) with relatively large dimensions $\{n=75-deg|$ for $deg\in \{0,1,2,3\}$. First, the left figure suggests that the optimal dimension and degree are $n=13$ and $deg=2$, where the 2-Wasserstein distance reaches minimum among all cases, and at the same time as the $L^2(\rhoT^L)$ error. For the other degrees, the minimum 2-Wasserstein distances are either reached before of after the $L^2(\rhoT^L)$ error. Thus, the 2-Wasserstein distance correctly selects the optimal dimension and degree for the hypothesis space. Second, (a) shows that the CEDR algorithm can effectively select the dimension range.  With the threshold in \eqref{eq:N2} being $\tau = 1.60$, which is relatively large (representing a tolerance of 100\% relative error), the dimension upper bounds are around $N=60$ for all degrees, and the ranges encloses the optimal dimensions selected by the 2-Wasserstein distance in (b).

Here we used a relatively large threshold for a rough estimation of the range of dimension. Clearly, our cross-validating error indicator $g(k)$ in \eqref{eq:err_gk} provides rich information about the increase of sampling error as the dimension increases. 
 A future direction is to extract the information, along with the decay of the integral operator, to find the trade-off between sampling error and approximation error.

\ifjournal \section*{Acknowledgments} \fi
\ifarXiv \paragraph{Acknowledgments} \fi
 MM, YGK and FL are partially supported by DE-SC0021361 and FA9550-21-1-0317. FL is partially funded by the NSF Award DMS-1913243.  
\ifjournal 
      \bibliographystyle{siamplain}
\fi
\ifarXiv
    \bibliographystyle{myplain}
\fi
{\small
\bibliography{ref_FeiLU2021_6,ref_myLib,ref_Bspline,latentDS,LearningTheory,ref_regularization}
}

\end{document}